\documentclass[11pt]{article}

\usepackage{fullpage,times,url,bm}

\usepackage{amsthm,amsfonts,amsmath,amssymb,epsfig,color,float,graphicx,verbatim}
\usepackage{algorithm,algorithmic}
\usepackage{bbm}
\usepackage{natbib}

\usepackage{graphicx}
\usepackage{subcaption}

\usepackage{array}

\usepackage{hyperref}
\hypersetup{
	colorlinks   = true, %Colours links instead of ugly boxes
	urlcolor     = blue, %Colour for external hyperlinks
	linkcolor    = blue, %Colour of internal links
	citecolor   = black %Colour of citations
}

\newtheorem{theorem}{Theorem}[section]

\newtheorem{lemma}{Lemma}[section]

\newtheorem{remark}{Remark}[section]

\usepackage{amssymb}

\usepackage{dsfont}
\newcommand{\onefunc}{\mathds{1}}

\newcommand{\stam}[1]{}

\usepackage{mathtools}

\newcommand{\bx}{\mathbf{x}}
\newcommand{\bw}{\mathbf{w}}

\newcommand{\bu}{\mathbf{u}}
\newcommand{\bv}{\mathbf{v}}

\newcommand{\bh}{\mathbf{h}}

\newcommand{\bbeta}{\boldsymbol{\beta}}

\newcommand{\btheta}{{\boldsymbol{\theta}}}

\newcommand{\cl}{{\cal L}}

\newcommand{\cn}{{\cal N}}

\DeclareMathOperator*{\argmax}{argmax}
\DeclareMathOperator*{\argmin}{argmin}

\newcommand{\reals}{{\mathbb R}}

\newcommand{\zero}{{\mathbf{0}}}

\newcommand{\diag}{\mathrm{diag}}

\newcommand{\inner}[1]{\langle #1 \rangle}
\newcommand{\norm}[1]{\left\|#1\right\|}
\newcommand{\snorm}[1]{\|#1\|} %small norm

\title{On Margin Maximization in Linear and ReLU Networks}

\author{
	Gal Vardi\thanks{Toyota Technological Institute at Chicago and the Hebrew University of Jerusalem, \texttt{galvardi@ttic.edu}. Work done while the author was at the Weizmann Institute of Science}
	\and
	Ohad Shamir\thanks{Weizmann Institute of Science, Israel, \texttt{ohad.shamir@weizmann.ac.il}}
	\and
	Nathan Srebro\thanks{Toyota Technological Institute at Chicago, \texttt{nati@ttic.edu}}	
}

\date{}

\begin{document}

\maketitle

\begin{abstract}
The implicit bias of neural networks has been extensively studied in recent years. \cite{lyu2019gradient} showed that in homogeneous networks trained with the exponential or the logistic loss, gradient flow converges to a KKT point of the max margin problem in parameter space. However, that leaves open the question of whether this point will generally be an actual optimum of the max margin problem. In this paper, we study this question in detail, for several neural network architectures involving linear and ReLU activations. Perhaps surprisingly, we show that in many cases, the KKT point is not even a \emph{local} optimum of the max margin problem. On the flip side, we identify 
multiple 
settings where a local or global optimum can be guaranteed. 
%Finally, we answer a question posed in \cite{lyu2019gradient} by showing that for \emph{non-homogeneous} networks, the normalized margin may strictly decrease over time.
\end{abstract}

\section{Introduction}

A central question in the theory of deep learning is how neural networks generalize even when trained without any explicit regularization, and when there are far more learnable parameters than training examples. In such optimization problems there are many solutions that label the training data correctly, and gradient descent seems to prefer solutions that generalize well \citep{zhang2016understanding}. Hence, it is believed that gradient descent induces an {\em implicit bias} \citep{neyshabur2014search,neyshabur2017exploring}, and characterizing this bias has been a subject of extensive research in recent years. 

A main focus in the theoretical study of implicit bias is on \emph{homogeneous} neural networks. These are networks where scaling the parameters by any factor $\alpha>0$ scales the predictions by $\alpha^L$ for some constant $L$. For example, fully-connected and convolutional ReLU networks without bias terms are homogeneous. \cite{lyu2019gradient} proved that in linear and ReLU homogeneous networks trained with the exponential or the logistic loss, if gradient flow converges to a sufficiently small loss\footnote{They also assumed directional convergence, but \citep{ji2020directional} later showed that this assumption is not required.}, then the direction to which the parameters of the network converge can be characterized as a first order stationary point (KKT point) of the maximum margin problem in parameter space. 
%if gradient flow converges to zero loss, then the direction to which gradient flow converges can be characterized as a first order stationary point (KKT point) of the maximum margin problem in parameter space. 
Namely,  the problem of minimizing the $\ell_2$ norm of the parameters under the constraints that each training example is classified correctly with margin at least $1$.
They also showed that this KKT point satisfies necessary conditions for optimality. However, the conditions are not known to be sufficient even for local optimality. It is analogous to showing that some unconstrained optimization problem converges to a point with gradient zero, without proving that it is either a global or a local minimum.
Thus, the question of when gradient flow maximizes the margin remains open. Understanding margin maximization may be crucial for explaining generalization in deep learning, and it might allow us to utilize margin-based generalization bounds for neural networks.

In this work we consider several architectures of homogeneous neural networks with linear and ReLU activations, and study whether the aforementioned KKT point is guaranteed to be a global optimum of the maximum margin problem, a local optimum, or neither. Perhaps surprisingly, our results imply that in many cases, such as depth-$2$ fully-connected ReLU networks and depth-$2$ diagonal linear networks, the KKT point may not even be a \emph{local} optimum of the maximum-margin problem. On the flip side, we identify multiple settings where a local or global optimum can be guaranteed.

We now describe our results in a bit more detail. We denote by $\cn$ the class of neural networks without bias terms, where the weights in each layer might have an arbitrary sparsity pattern, and weights might be shared\footnote{See Section~\ref{sec:preliminaries} for the formal definition.}. The class $\cn$ contains, for example, \emph{convolutional networks}. Moreover, we denote by $\cn_{\text{no-share}}$ the subclass of $\cn$ that contains only networks without shared weights, such as \emph{fully-connected networks} and \emph{diagonal networks} (cf. \cite{gunasekar2018bimplicit,yun2020unifying}). We describe our main results below, and also summarize them in  Tables~\ref{table:1} and~\ref{table:2}.

\vspace{0.3cm}
{\bf Fully-connected networks:}
\begin{itemize}
	\item In linear fully-connected networks of any depth the KKT point is a global optimum\footnote{We note that margin maximization for such networks in predictor space is already known \citep{ji2020directional}. However, margin maximization in predictor space does not necessarily imply margin maximization in parameter space.}.
	
	\item In fully-connected depth-$2$ ReLU networks the KKT point may not even be a local optimum. Moreover, this negative result holds with constant probability over the initialization, i.e., there is a training dataset such that gradient flow with random initialization converges with positive probability to the direction of a KKT point which is not a local optimum.
\end{itemize}

{\bf Depth-$2$ networks in $\cn$:}
\begin{itemize}
	\item %The positive result on fully-connected linear networks does not extend to networks with sparse weights: 
	In linear 
	%diagonal networks 
	networks with sparse weights, and specifically in diagonal networks,
	we show that
	the KKT point may not be a local optimum. 

	\item In our proof of the above negative result, the KKT point contains a neuron whose weights vector is zero. However, in practice gradient descent often converges to networks that do not contain such zero neurons. We show that for linear networks in $\cn_{\text{no-share}}$, if the KKT point has only non-zero weights vectors, then it is a global optimum. Thus, despite the above negative result, a reasonable assumption on the KKT point allows us to obtain a strong positive result.
	We also show 
	%that even for the simple case of depth-$2$ diagonal linear networks, the optimality of the KKT points can be unexpectedly subtle, in the context of margin maximization in predictor space  (see Remark~\ref{rem:implications on predictor}).
	some implications of our results on margin maximization in predictor space for depth-$2$ diagonal linear networks (see Remark~\ref{rem:implications on predictor}).
	
	\item For ReLU networks in $\cn_{\text{no-share}}$, in order to obtain a positive result we need a stronger assumption. We show that if the KKT point is such that for every input in the dataset the input to every hidden neuron in the network is non-zero, then it is guaranteed to be a local optimum (but not necessarily a global optimum).
	%We also show that the assuming that the network does not have shared weights is indeed required here, since for convolutional networks this positive result no longer holds.
	
	\item %For linear or ReLU convolutional networks, even if the above assumptions hold, the KKT point may not be a local optimum.
	We prove that assuming the network does not have shared weights
	is indeed required in the above positive results, since for networks with shared weights (such as convolutional networks) they no longer hold.
\end{itemize}

\newpage
{\bf Deep networks in $\cn$:}
\begin{itemize}
	%\item We show that the positive results on depth-$2$ linear and ReLU networks in $\cn_{\text{no-share}}$ (under the assumptions described above) do not apply to deeper networks.
	
	\item %After discussing the difficulty in extending our positive results to deeper networks, 
	We discuss the difficulty in extending our positive results to deeper networks. Then,
	we study a weaker notion of margin maximization: maximizing the margin for each layer separately. For linear networks of depth $m \geq 2$ in $\cn$ (including networks with shared weights), we show that the KKT point is a global optimum of the per-layer maximum margin problem. For ReLU networks the KKT point may not even be a local optimum of this problem, but under the assumption on non-zero inputs to all neurons it is a local optimum.
\end{itemize}

As detailed above, we consider several different settings, and the results vary dramatically between the settings. Thus, our results draw a somewhat complicated picture. 
%We believe that in order to understand margin maximization in deep learning such a detailed study on the different settings is indeed required. 
Overall, our negative results show that even in very simple settings gradient flow does not maximize the margin even locally, and we believe that these results should be used as a starting point for studying which assumptions are required for proving margin maximization. 
Our positive results indeed show that under certain reasonable assumptions gradient flow maximizes the margin (either locally or globally). Also, the notion of per-layer margin maximization which we consider suggests another path for obtaining positive results on the implicit bias.

In the paper, our focus is on understanding what can be guaranteed for the KKT convergence points specified in \cite{lyu2019gradient}. Accordingly, in most of our negative results, the construction assumes some specific initialization of gradient flow, and does not quantify how ``likely'' they are to be reached under some random initialization. An exception is our negative result for depth-$2$ fully-connected ReLU networks (Theorem~\ref{thm:depth 2 relu negative}), which holds with constant probability under reasonable random initializations. Understanding whether this can be extended to the other settings we consider is an interesting problem for future research.
% Our negative result for depth-$2$ fully-connected ReLU networks holds with positive probability over random initializations. However, in the other negative results we consider specific initializations. Thus, these results imply that the geometry of the loss function is such that for some initializations gradient flow converges in direction to a KKT point which is not a local optimum. We believe that studying the probability over random initializations of converging to such bad KKT points is an important problem for future research.

\stam{
Finally, we consider \emph{non-homogeneous networks}, for example, networks with skip connections or bias terms. \cite{lyu2019gradient} showed that a smoothed version of the \emph{normalized margin} is monotonically increasing when training homogeneous networks. They observed empirically that the normalized margin is monotonically increasing also when training non-homogeneous networks, but did not provide a proof for this phenomenon and left it as an open problem. We give an example for a simple non-homogeneous network where the normalized margin (as well as the smoothed margin) is strictly \emph{decreasing} (see Theorem~\ref{thm:non-homogeneous}).
}%stam

Our paper is structured as follows: In Section \ref{sec:preliminaries} we provide necessary notations and definitions, and discuss relevant prior results. Additional related works are discussed in Appendix~\ref{app:more related}. In Sections~\ref{sec:fully-connected}, \ref{sec:depth 2} and~\ref{sec:deep} we state our results on fully-connected networks, depth-$2$ networks in $\cn$ and deep networks in $\cn$ respectively, and provide some proof ideas. 
%In Section~\ref{sec:non homogeneous} we state our result on non-homogeneous networks. 
All formal proofs are deferred to Appendix~\ref{app:proofs}. 
%We conclude with a short discussion (Section~\ref{sec:discussion}). 

\begin{table}[t]%[h!]
\begin{center}
\begin{tabular}{|m{10em} || m{5em} | m{5em}|} 
 \hline
  & Linear & ReLU  \\ [0.5ex] 
 \hline\hline
 Fully-connected & Global (Thm.~\ref{thm:deep positive linear})& Not local~$\text{      }$ (Thm.~\ref{thm:depth 2 relu negative})  \\ 
  \hline
 $\cn_{\text{no-share}}$ & Not local~$\text{      }$ (Thm.~\ref{thm:depth 2 linear negative})& Not local~$\text{      }$ (Thm.~\ref{thm:depth 2 relu negative}) \\
 \hline
 $\cn_{\text{no-share}}$ assuming non-zero weights vectors & Global~$\text{      }$ (Thm.~\ref{thm:depth 2 linear}) & Not local~$\text{      }$ (Thm.~\ref{thm:depth 2 linear}) \\
 \hline
  $\cn_{\text{no-share}}$ assuming non-zero inputs to all neurons  & Global (Thm.~\ref{thm:depth 2 linear})& Local,\;\;\; Not global~$\text{      }$ (Thm.~\ref{thm:depth 2 relu positive}) \\
  \hline
 $\cn$ assuming non-zero inputs to all neurons & Not local~$\text{      }$ (Thm.~\ref{thm:depth 2 cnn}) & Not local~$\text{      }$ (Thm.~\ref{thm:depth 2 cnn})  \\
   \hline
\end{tabular}
\end{center}
\caption{Results on depth-$2$ networks.}
\label{table:1}
\end{table}

\begin{table}[t]
\begin{center}
\begin{tabular}{|m{12em} || m{5em} | m{5em}|} 
 \hline
  & Linear & ReLU  \\ [0.5ex] 
 \hline\hline
 Fully-connected & Global (Thm.~\ref{thm:deep positive linear}) & Not local~$\text{      }$  (Thm.~\ref{thm:depth 2 relu negative}) \\ 
 \hline 
 $\cn_{\text{no-share}}$ assuming non-zero inputs to all neurons & Not local~$\text{      }$ (Thm.~\ref{thm:deep negative}) & Not local~$\text{      }$  (Thm.~\ref{thm:deep negative}) \\
 \hline
 $\cn$ - max margin for each layer separately   & Global (Thm.~\ref{thm:positive each layer linear}) & Not local~$\text{      }$  (Thm.~\ref{thm:negative each layer relu}) \\
 \hline
 $\cn$ - max margin for each layer separately, assuming non-zero inputs to all neurons  & Global (Thm.~\ref{thm:positive each layer linear}) & Local,\;\;\; Not global~$\text{      }$  (Thm.~\ref{thm:positive each layer relu}) \\
  \hline
\end{tabular}
\end{center}
\caption{Results on deep networks.}
\label{table:2}
\end{table}

%subsection*{Related work}

\section{Preliminaries}
\label{sec:preliminaries}

\paragraph{Notations.}

We use bold-faced letters to denote vectors, e.g., $\bx=(x_1,\ldots,x_d)$. For $\bx \in \reals^d$ we denote by $\norm{\bx}$ the Euclidean norm.
We denote by $\onefunc(\cdot)$ the indicator function, for example $\onefunc(t \geq 5)$ equals $1$ if $t \geq 5$ and $0$ otherwise.
For an integer $d \geq 1$ we denote $[d]=\{1,\ldots,d\}$.

\paragraph{Neural networks.}

A {\em fully-connected neural network} $\Phi$ of depth $m \geq 2$ is parameterized by a collection $\btheta = [W^{(l)}]_{l=1}^m$ of weight matrices, such that for every layer $l \in [m]$ we have $W^{(l)} \in \reals^{d_l \times d_{l-1}}$. Thus, $d_l$ denotes the number of neurons in the $l$-th layer (i.e., the \emph{width} of the layer). We assume that $d_m=1$ and denote by $d:=d_0$ the input dimension. 
The neurons in layers $[m-1]$ are called \emph{hidden neurons}.
A fully-connected network computes a function $\Phi(\btheta; \cdot): \reals^d \to \reals$ defined recursively as follows. For an input $\bx \in \reals^d$ we set $\bh'_0=\bx$, and define for every $j \in [m-1]$ the input to the $j$-th layer as $\bh_j =W^{(j)} \bh'_{j-1}$, and the output of the $j$-th layer as $\bh'_j = \sigma(\bh_j)$, where $\sigma:\reals \to \reals$ is an activation function that acts coordinate-wise on vectors. Then, we define $\Phi(\btheta; \bx) = W^{(m)} \bh'_{m-1}$. Thus, there is no activation function in the output neuron. When considering depth-$2$ fully-connected networks we often use a parameterization  $\btheta = [\bw_1,\ldots,\bw_k,\bv]$ where $\bw_1,\ldots,\bw_k$ are the weights vectors of the $k$ hidden neurons (i.e., correspond to the rows of the first layer's weight matrix) and $\bv$ are the weights of the second layer.

We also consider neural networks where some weights can be missing or shared.
We define a class $\cn$ of networks that may contain sparse and shared weights as follows.
%A {\em neural network with sparse and shared weights}
A network $\Phi$ in $\cn$ is parameterized by $\btheta = [\bu^{(l)}]_{l=1}^m$ where $m$ is the depth of $\Phi$, and $\bu^{(l)} \in \reals^{p_l}$ are the parameters of the $l$-th layer. We denote by $W^{(l)} \in \reals^{d_l \times d_{l-1}}$ the weight matrix of the $l$-th layer.
%, where $d_l$ is the width of the layer.  and $d_0=d$ is the input dimension. 
The matrix $W^{(l)}$ is described by the vector $\bu^{(l)}$, 
%which consists of the {\em free parameters} in the $l$-th layer, 
and a function $g_l:[d_l] \times [d_{l-1}] \to [p_l] \cup \{0\}$ as follows: $W^{(l)}_{i j} = 0$ if $g_l(i,j)=0$, and $W^{(l)}_{i j} = u_k$ if $g_l(i,j)=k>0$. 
Thus, the function $g_l$ represents the sparsity and weight-sharing pattern of the $l$-th layer, and the dimension $p_l$ of $\bu^{(l)}$ is the number of free parameters in the layer.
We denote by $d:=d_0$ the input dimension of the network and assume that the output dimension $d_m$ is $1$.
The function $\Phi(\btheta; \cdot): \reals^d \to \reals$ computed by the neural network is defined recursively by the weight matrices as in the case of fully-connected networks.
%A convolutional network is an example for a network with sparse and shared weights.
For example, convolutional neural networks are in $\cn$.
Note that the networks in $\cn$ do not have bias terms and do not allow weight sharing between different layers.
Moreover, we define a subclass $\cn_{\text{no-share}}$ of $\cn$, that contains networks without shared weights.
Formally, 
a network $\Phi$ is in $\cn_{\text{no-share}}$ 
if for every layer $l$ and every $k \in  [p_l]$ there is at most one  $(i,j) \in [d_l] \times [d_{l-1}] $ such that $g_l(i,j)=k$. Thus, networks in $\cn_{\text{no-share}}$ might have sparse weights, but do not allow shared weights.
For example, diagonal networks
%\footnote{See Section~\ref{sec:deep} for a formal definition.} 
(defined below)
and fully-connected networks are 
in $\cn_{\text{no-share}}$.
%networks with sparse weights.

A {\em diagonal neural network} is a network 
%with sparse weights 
in $\cn_{\text{no-share}}$
such that the weight matrix of each layer is diagonal, except for the last layer. Thus, the network is parameterized by $\btheta =  [\bw_1,\ldots,\bw_m]$ where $\bw_j \in \reals^d$ for all $j \in [m]$, and it computes a function $\Phi(\btheta; \cdot): \reals^d \rightarrow \reals$ defined recursively as follows. For an input $\bx \in \reals^d$ set $\bh_0 = \bx$. For $j \in [m-1]$, the output of the $j$-th layer is $\bh_j = \sigma(\diag(\bw_j) \bh_{j-1})$. Then, we have $\Phi(\btheta; \bx) = \bw_m^\top \bh_{m-1}$.

In all the above definitions the parameters $\btheta$ of the neural networks are given by a collection of matrices or vectors. We often view $\btheta$ as the vector obtained by concatenating the matrices or vectors in the collection. Thus, $\norm{\btheta}$ denotes the $\ell_2$ norm of the vector $\btheta$.

The ReLU activation function is defined by $\sigma(z) = \max\{0,z\}$, and the linear activation is $\sigma(z)=z$. In this work we focus on ReLU networks (i.e., networks where all neurons have the ReLU activation) and on linear networks (where all neurons have the linear activation). 
We say that a network $\Phi$ is \emph{homogeneous} if there exists $L>0$ such that for every $\alpha>0$ and $\btheta,\bx$ we have $\Phi(\alpha \btheta; \bx) = \alpha^L \Phi(\btheta; \bx)$. 
Note that in our definition of 
%networks with sparse and shared weights 
the class $\cn$
we do not allow bias terms, and hence all linear and ReLU networks 
%with sparse and shared weights 
in $\cn$
are homogeneous, where $L$ is the depth of the network. 
%With the exception of Section~\ref{sec:non homogeneous} which studies non-homogeneous networks, all 
All
networks considered in this work are homogeneous.

\paragraph{Optimization problem and gradient flow.}

Let $S = \{(\bx_i,y_i)\}_{i=1}^n \subseteq \reals^d \times \{-1,1\}$ be a binary classification training dataset. Let $\Phi$ be a neural network parameterized by $\btheta \in \reals^m$. 
%The output of $\Phi$ on an input $\bx \in \reals^d$ is denoted by $\Phi(\btheta; \bx)$. 
For a loss function $\ell:\reals \to \reals$ the empirical loss of $\Phi(\btheta; \cdot)$ on the dataset $S$ is 
\begin{equation}
\label{eq:objective}
	\cl(\btheta) := \sum_{i=1}^n \ell(y_i \Phi(\btheta; \bx_i))~.
\end{equation} 
%We denote $q_i(\btheta) = y_i \Phi(\btheta; \bx_i)$.
We focus on the exponential loss $\ell(q) = e^{-q}$ and the logistic loss $\ell(q) = \log(1+e^{-q})$.

We consider gradient flow on the objective given in Eq.~\ref{eq:objective}. This setting captures the behavior of gradient descent with an infinitesimally small step size. Let $\btheta(t)$ be the trajectory of gradient flow. Starting from an initial point $\btheta(0)$, the dynamics of $\btheta(t)$ is given by the differential equation $\frac{d \btheta(t)}{dt} = -\nabla \cl(\btheta(t))$.
Note that the ReLU function is not differentiable at $0$. Practical implementations of gradient methods define the derivative $\sigma'(0)$ to be some constant in $[0,1]$. We note that the exact value of $\sigma'(0)$ has no effect on our results.

\paragraph{Convergence to a KKT point of the maximum-margin problem.}

We say that a trajectory $\btheta(t)$ {\em converges in direction} to $\tilde{\btheta}$ if 
%there is some $\alpha > 0$ such that 
$\lim_{t \to \infty}\frac{\btheta(t)}{\norm{\btheta(t)}} = \frac{\tilde{\btheta}}{\norm{\tilde{\btheta}}}$.
%\alpha \tilde{\btheta}$.
Throughout this work we use the following theorem:

\begin{theorem}[Paraphrased from \cite{lyu2019gradient,ji2020directional}]
\label{thm:known KKT}
	Let $\Phi$ be a homogeneous linear or ReLU neural network. Consider minimizing either the exponential or the logistic loss over a binary classification dataset $ \{(\bx_i,y_i)\}_{i=1}^n$ using gradient flow. Assume that there exists time $t_0$ such that $\cl(\btheta(t_0))<1$, namely, $\Phi$ classifies every $\bx_i$ correctly. Then, gradient flow converges in direction to a first order stationary point (KKT point) of the following maximum margin problem in parameter space:
\begin{equation}
\label{eq:optimization problem}
	\min_\btheta \frac{1}{2} \norm{\btheta}^2 \;\;\;\; \text{s.t. } \;\;\; \forall i \in [n] \;\; y_i \Phi(\btheta; \bx_i) \geq 1~.
\end{equation}
Moreover, $\cl(\btheta(t)) \to 0$ and $\norm{\btheta(t)} \to \infty$ as $t \to \infty$.
\end{theorem}

In the case of ReLU networks, Problem~\ref{eq:optimization problem} is non-smooth. Hence, the KKT conditions are defined using the Clarke subdifferential, which is a generalization of the derivative for non-differentiable functions. See Appendix~\ref{app:KKT} for a formal definition. 
We note that  \cite{lyu2019gradient} proved the above theorem under the assumption that $\btheta$ converges in direction, and \cite{ji2020directional} showed that such a directional convergence occurs and hence this assumption is not required. 

\cite{lyu2019gradient} also showed that the KKT conditions of Problem~\ref{eq:optimization problem} are necessary for 
%global 
optimality. 
In convex optimization problems, necessary KKT conditions are also sufficient for global optimality. However, the constraints in Problem~\ref{eq:optimization problem} are highly non-convex. 
Moreover, the standard method for proving that necessary KKT conditions are sufficient for \emph{local} optimality, is by showing that the KKT point satisfies certain \emph{second order sufficient conditions (SOSC)} (cf. \cite{ruszczynski2011nonlinear}). However, even when $\Phi$ is a linear neural network it is not known when such conditions hold.
Thus, the KKT conditions of Problem~\ref{eq:optimization problem} are not known to be sufficient even for local optimality. 

A linear network with weight matrices  $W^{(1)}, \ldots, W^{(m)}$ computes a linear predictor $\bx \mapsto \inner{\bbeta,\bx}$ where $\bbeta = W^{(m)} \cdot \ldots \cdot W^{(1)}$. Some prior works studied the implicit bias of linear networks in the \emph{predictor space}. Namely, characterizing the vector $\bbeta$ from the aforementioned linear predictor. \cite{gunasekar2018bimplicit} studied the implications of margin maximization in the \emph{parameter space} on the implicit bias in predictor space. They showed that minimizing $\norm{\btheta}$ (under the constraints in Problem~\ref{eq:optimization problem}) implies: (1) Minimizing $\norm{\bbeta}_2$ for fully-connected linear networks; (2) Minimizing $\norm{\bbeta}_{2/L}$ for diagonal linear networks of depth $L$; (3) Minimizing $\snorm{\hat{\bbeta}}_{2/L}$ for linear convolutional networks of depth $L$ with full-dimensional convolutional filters, where $\hat{\bbeta}$ are the Fourier coefficients of $\bbeta$.
%\footnote{The implications of margin maximization in parameter space on the implicit bias in predictor space for linear convolutional networks was further studied in \cite{jagadeesan2021inductive}.}
However, these implications may not hold if gradient flow converges to a KKT point which is not a global optimum of Problem~\ref{eq:optimization problem}.

For some classes of linear networks, positive results were obtained directly in predictor space, without assuming convergence to a global optimum of Problem~\ref{eq:optimization problem} in parameter space. Most notably, for fully-connected linear networks (of any depth), \cite{ji2020directional} showed that under the assumptions of Theorem~\ref{thm:known KKT},
%\footnote{Similar results under stronger assumptions were established in \cite{gunasekar2018bimplicit} and in \cite{ji2018gradient}.} 
gradient flow maximizes the $\ell_2$ margin in predictor space. Note that margin maximization in predictor space does not necessarily imply margin maximization in parameter space. 
Moreover, some results on the implicit bias in predictor space of linear convolutional networks with full-dimensional convolutional filters are given in  \cite{gunasekar2018bimplicit}. 
%We note that such convolutional networks are different from the convolutional networks that we consider in this work, which do not have full-dimensional convolutional filters. 
However, the architecture and set of assumptions are different than what we focus on.
See Appendix~\ref{app:more related} for a discussion on additional related work.

\stam{
\paragraph{Normalized margin}

We define the \emph{normalized margin} as follows:
\begin{equation}
\label{eq:normalized margin}
	\bar{\gamma}(\btheta) := \min_{i \in [n]} y_i \Phi \left(\frac{\btheta}{\norm{\btheta}}; \bx_i \right)~.
\end{equation}
If $\Phi$ is homogeneous then maximizing the normalized margin is equivalent to solving Problem~\ref{eq:optimization problem}, i.e., minimizing $\norm{\btheta}$ under the constraints (cf. \cite{lyu2019gradient}). 
}%stam

\section{Fully-connected networks}
\label{sec:fully-connected}

First, we show that fully-connected linear networks of any depth converge in direction to a global optimum of Problem~\ref{eq:optimization problem}. 

\begin{theorem}
\label{thm:deep positive linear}
	Let $m \geq 2$ and let $\Phi$ be a depth-$m$ fully-connected linear network parameterized by $\btheta$. 
	% parameterized by $\btheta = [W^{(1)},\ldots,W^{(m)}]$. 
	Consider minimizing either the exponential or the logistic loss over a dataset $ \{(\bx_i,y_i)\}_{i=1}^n$ using gradient flow. Assume that there exists time $t_0$ such that $\cl(\btheta(t_0))<1$. Then, gradient flow converges in direction to a global optimum of Problem~\ref{eq:optimization problem}.
\end{theorem}
\begin{proof}[Proof idea (for the complete proof see Appendix~\ref{app:proof of deep positive linear})]
 	Building on results from \cite{ji2020directional} and \cite{du2018algorithmic}, we show that gradient flow converges in direction to a KKT point  $\tilde{\btheta} = [\tilde{W}^{(1)},\ldots,\tilde{W}^{(m)}]$ such that for every $l \in [m]$ we have $\tilde{W}^{(l)} = C \cdot \bv_l \bv_{l-1}^\top$, where $C>0$ and $\bv_0,\ldots,\bv_m$ are unit vectors (with $\bv_m = 1$). Also, we have $\snorm{\tilde{W}^{(m)} \cdot \ldots \cdot \tilde{W}^{(1)}} = C^m =  \min \norm{\bu} \text{ s.t. } y_i \bu^\top \bx_i \geq 1 \text{ for all } i \in [n]$.
	Then, we show that every $\btheta$ that satisfies these properties, and satisfies the constraints of Problem~\ref{eq:optimization problem}, is a global optimum.
	Intuitively, the most ``efficient" way (in terms of minimizing the parameters) to achieve margin $1$ with a linear fully-connected network, is by using a network such that the direction of its corresponding linear predictor maximizes the margin, the layers are balanced (i.e., have equal norms), and the weight matrices of the layers are aligned.
\end{proof}

We now prove that the positive result in Theorem~\ref{thm:deep positive linear} does not apply to ReLU networks.
We show that in depth-$2$ fully-connected ReLU networks gradient flow might converge in direction to a KKT point of Problem~\ref{eq:optimization problem} which is not even a local optimum. Moreover, it occurs under conditions holding with constant probability over reasonable random initializations.

%stam

\begin{theorem}
\label{thm:depth 2 relu negative}
	Let $\Phi$ be a depth-$2$ fully-connected ReLU network with input dimension $2$ and two hidden neurons. Namely, for $\btheta = [\bw_1,\bw_2,\bv]$ and $\bx \in \reals^2$ we have $\Phi(\btheta; \bx) = \sum_{l=1}^2 v_l \sigma(\bw_l^\top \bx)$. Consider minimizing either the exponential or the logistic loss using gradient flow. 
	Consider the dataset  $\{(\bx_1,y_1),(\bx_2,y_2)\}$ where $\bx_1 = \left(1,\frac{1}{4}\right)^\top$, $\bx_2 = \left(-1,\frac{1}{4}\right)^\top$, and $y_1=y_2=1$.
	Assume that the initialization $\btheta(0)$ is such that  for every $i \in \{1,2\}$ we have $\inner{\bw_1(0),\bx_i}>0$ and $\inner{\bw_2(0),\bx_i}<0$. Thus, the first hidden neuron is active for both inputs, and the second hidden neuron is not active. Also, assume that $v_1(0) > 0$.
	Then, gradient flow converges to zero loss, and converges in direction to a KKT point of Problem~\ref{eq:optimization problem} which is not a local optimum.
\end{theorem}
\begin{proof}[Proof idea (for the complete proof see Appendix~\ref{app:proof of depth 2 relu negative})]
	By analyzing the dynamics of gradient flow on the given dataset, we show that it converges to zero loss, and converges in direction to a KKT point $\tilde{\btheta}$ such that $\tilde{\bw}_1 = (0,2)^\top$, $\tilde{v}_1 = 2$, $\tilde{\bw}_2 = \zero$, and $\tilde{v}_2=0$. Note that $\tilde{\bw}_2 = \zero$ and $\tilde{v}_2=0$ since $\bw_2(t), v_2(t)$ remain constant during the training and $\lim_{t \to \infty}\norm{\btheta(t)}=\infty$. See Figure~\ref{fig:negative relu} for an illustration.
	Then, we show that for every $0<\epsilon<1$ there exists some $\btheta'$ such that $\snorm{\btheta' - \tilde{\btheta}} \leq \epsilon$, $\btheta'$ satisfies $y_i \Phi(\btheta'; \bx_i) \geq 1$ for every $i \in \{1,2\}$, and $\snorm{\btheta'} < \snorm{\tilde{\btheta}}$. Such $\btheta'$ is obtained from $\tilde{\btheta}$ by slightly changing $\tilde{\bw}_1$, $\tilde{\bw}_2$, and $\tilde{v}_2$. Thus, by using the second hidden neuron, which is not active in $\tilde{\btheta}$, we can obtain a solution $\btheta'$ with smaller norm.
\end{proof}

\begin{figure}[t]
	\centering
	\includegraphics[scale=0.6]{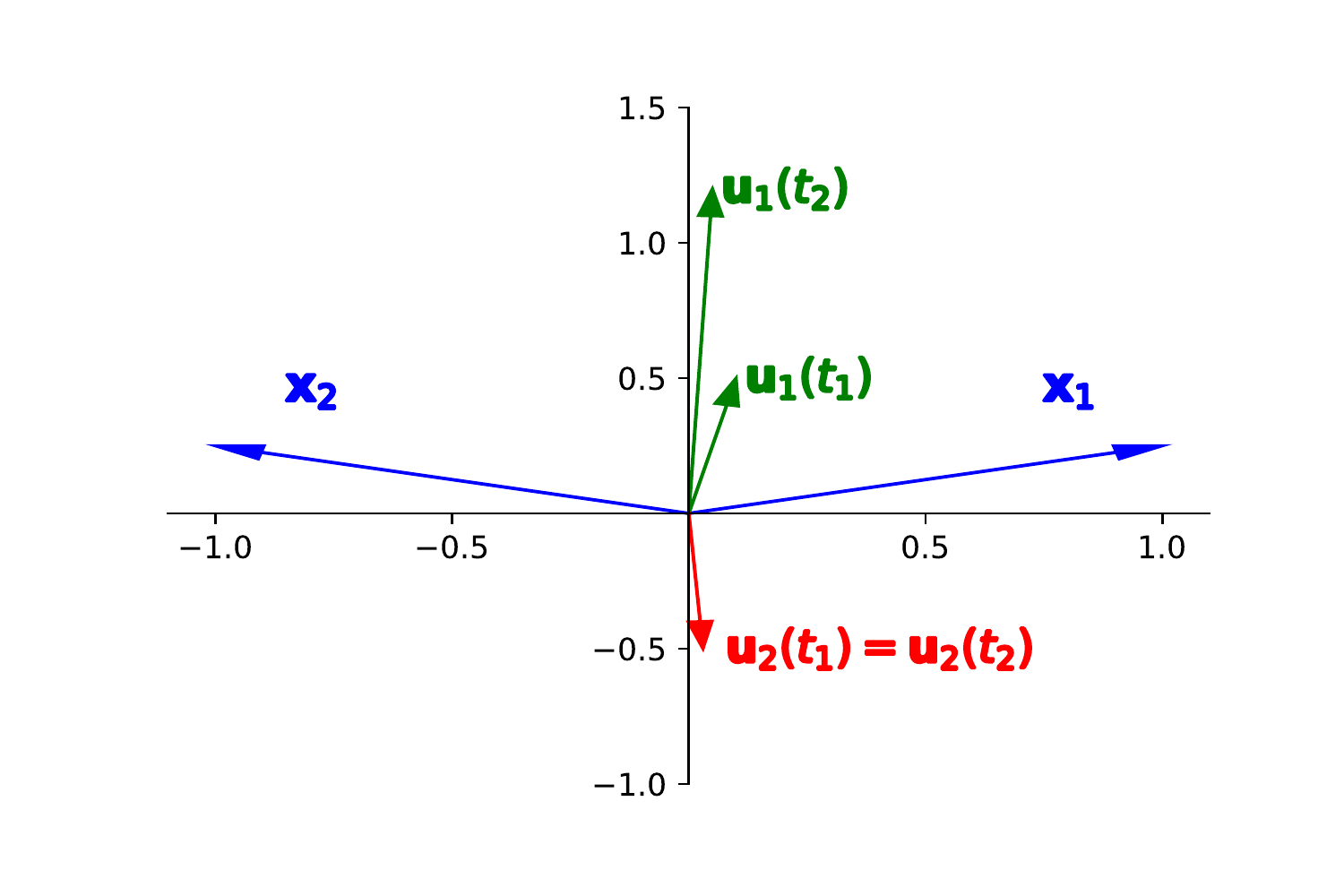}
	\caption{For time $t$ we denote $\bu_i(t) = v_i(t) \bw_i(t)$. Above we illustrate $\bu_1(t),\bu_2(t)$ for times $t_1 < t_2$.
	As $t \to \infty$ we have $\norm{\bu_1(t)} \to \infty$ and $\bu_1$ converges in direction to $(0,1)$. The vector $\bu_2$ remains constant during the training. 
	%This is illustrated above for times $t_1 < t_2$. 
	Hence $\frac{\bu_2(t)}{\norm{\btheta(t)}} \to \zero$.} 
	\label{fig:negative relu}
\end{figure}

We note that the assumption on the initialization in the above theorem holds with constant probability for standard initialization schemes (e.g., Xavier initialization).

%We note that the above theorem is stated w.r.t. the Xavier initialization, but holds also for other random initialization schemes. 

\begin{remark}[Unbounded sub-optimality]
	By choosing appropriate inputs $\bx_1,\bx_2$ in the setting of Theorem~\ref{thm:depth 2 relu negative}, it is not hard to show that the sub-optimality of the KKT point w.r.t. the global optimum can be arbitrarily large. Namely, for every large $M>0$ we can choose a dataset where the angle between $\bx_1$ and $\bx_2$ is sufficiently close to $\pi$, such that $\frac{\snorm{\tilde{\btheta}}}{\snorm{\btheta^*}} \geq M$, where $\tilde{\btheta}$ is a KKT point to which gradient flow converges, and $\btheta^*$ is a global optimum of Problem~\ref{eq:optimization problem}.
	Indeed, as illustrated in Figure~\ref{fig:negative relu}, if one neuron is active on both inputs and the other neuron is not active on any input, then the active neuron needs to be very large in order to achieve margin $1$, while if each neuron is active on a single input then we can achieve margin $1$ with much smaller parameters.
	We note that such unbounded sub-optimality can be obtained also in other negative results in this work (in Theorems~\ref{thm:depth 2 linear negative}, \ref{thm:depth 2 relu positive}, \ref{thm:depth 2 cnn} and~\ref{thm:positive each layer relu}).
\end{remark}

\begin{remark}[Robustness to small perturbations]
	  Theorem~\ref{thm:depth 2 relu negative} holds even if we slightly perturb the inputs $\bx_1,\bx_2$. Thus, it is not sensitive to small changes in the dataset. We note that such robustness to small perturbations can be shown also for the negative results in Theorems~\ref{thm:depth 2 linear negative}, \ref{thm:depth 2 relu positive}, \ref{thm:deep negative} and~\ref{thm:positive each layer relu}.
\end{remark}

\section{Depth-$2$ networks in $\cn$}
\label{sec:depth 2}

In this section we study depth-$2$ linear and ReLU networks in $\cn$. We first show that already for linear networks in $\cn_{\text{no-share}}$ (more specifically, for diagonal networks) gradient flow may not converge even to a local optimum.

\begin{theorem}
\label{thm:depth 2 linear negative}
	Let $\Phi$ be a depth-$2$ linear or ReLU diagonal neural network parameterized by $\btheta=[\bw_1,\bw_2]$. 
	Consider minimizing either the exponential or the logistic loss using gradient flow. There exists a dataset $\{(\bx,y)\} \subseteq \reals^2 \times \{-1,1\}$ of size $1$ and an initialization $\btheta(0)$, such that gradient flow converges to zero loss, and converges in direction to a KKT point $\tilde{\btheta}$ of Problem~\ref{eq:optimization problem} which is not a local optimum. 
\end{theorem}
\begin{proof}[Proof idea (for the complete proof see Appendix~\ref{app:proof of depth 2 linear negative})]
	Let $\bx = (1,2)^\top$ and $y=1$. Let $\btheta(0)$ such that $\bw_1(0) = \bw_2(0) = (1,0)^\top$. 
	%Thus, the second hidden neuron has $0$ in both its incoming and outgoing weights, and hence it remains inactive during the training. 
	Recalling that the diagonal network computes the function $\bx \mapsto (\bw_1 \circ \bw_2)^\top \bx$ (where $\circ$ is the entry-wise product), we see that the second coordinate remains inactive during training.
	It is not hard to show that gradient flow converges in direction to the KKT point $\tilde{\btheta}$ with $\tilde{\bw}_1 = \tilde{\bw}_2 = (1,0)^\top$. However, it is not a local optimum, since for every small $\epsilon>0$ the parameters $\btheta' = [\bw'_1,\bw'_2]$ with $\bw'_1 = \bw'_2 = \left(\sqrt{1-\epsilon}, \sqrt{\frac{\epsilon}{2}}\right)^\top$ satisfy the constraints of Problem~\ref{eq:optimization problem}, and we have $\snorm{\btheta'} < \snorm{\tilde{\btheta}}$.
\end{proof}

By Theorem~\ref{thm:depth 2 relu negative} fully-connected ReLU networks may not converge to a local optimum, and by Theorem~\ref{thm:depth 2 linear negative} linear (and ReLU) networks with sparse weights may not converge to a local optimum. In the proofs of both of these negative results, gradient flow converges in direction to a KKT point such that one of the weights vectors of the hidden neurons is zero. However, in practice gradient descent often converges to a network that does not contain such disconnected neurons. Hence, a natural question is whether the negative results hold also in networks that do not contain neurons whose weights vector is zero. 
In the following theorem we show that in linear networks such an assumption allows us to obtain a positive result. Namely, in depth-$2$ linear networks in $\cn_{\text{no-share}}$, if gradient flow converges in direction to a KKT point of Problem~\ref{eq:optimization problem} that satisfies this condition, then it is guaranteed to be a global optimum. However, we also show that in ReLU networks assuming that all neurons have non-zero weights is not sufficient.

%We show that in depth-$2$ linear networks gradient flow converges to a global optimum of Problem~\ref{eq:optimization problem}, even for networks in $\cn_{\text{no-share}}$. 

\begin{theorem}
\label{thm:depth 2 linear}
We have:
\begin{enumerate}
	\item 
	Let $\Phi$ be a depth-$2$ linear neural network in $\cn_{\text{no-share}}$ parameterized by $\btheta$. 	
	Consider minimizing either the exponential or the logistic loss over a dataset $ \{(\bx_i,y_i)\}_{i=1}^n$ using gradient flow. Assume that there exists time $t_0$ such that $\cl(\btheta(t_0))<1$,
%. Then, gradient flow converges in direction to a global optimum of Problem~\ref{eq:optimization problem}.
and let $\tilde{\btheta}$ be the KKT point of Problem~\ref{eq:optimization problem} such that $\btheta(t)$ converges to $\tilde{\btheta}$ in direction (such $\tilde{\btheta}$ exists by Theorem~\ref{thm:known KKT}).
	Assume that in the network parameterized by $\tilde{\btheta}$ all hidden neurons have non-zero incoming weights vectors.
	Then, $\tilde{\btheta}$ is a global optimum of Problem~\ref{eq:optimization problem}.
	
	\item 
	Let $\Phi$ be a fully-connected depth-$2$ ReLU network with input dimension $2$ and $4$ hidden neurons parameterized by $\btheta$. 
	Consider minimizing either the exponential or the logistic loss using gradient flow. 
	There exists a dataset and an initialization $\btheta(0)$, such that gradient flow converges to zero loss, and converges in direction to a KKT point $\tilde{\btheta}$ of Problem~\ref{eq:optimization problem}, which is not a local optimum, and in the network parameterized by $\tilde{\btheta}$ all hidden neurons have non-zero incoming weights.
\end{enumerate}
\end{theorem}
\begin{proof}[Proof idea (for the complete proof see Appendix~\ref{app:proof of depth 2 linear})]
We give here the proof idea for part (1).
Let $k$ be the width of the network. For every $j \in [k]$ we denote by $\bw_j$ the incoming weights vector to the $j$-th hidden neuron, and by $v_j$ the outgoing weight. Let $\bu_j = v_j \bw_j$. We consider an optimization problem over the variables $\bu_1,\ldots,\bu_k$ where the objective is to minimize $\sum_{j \in [k]}\norm{\bu_j}$ and the constrains correspond to the constraints of Problem~\ref{eq:optimization problem}. Let $\tilde{\btheta} = [\tilde{\bw}_1,\ldots,\tilde{\bw}_k,\tilde{\bv}]$ be the KKT point of Problem~\ref{eq:optimization problem} to which gradient flow converges in direction. For every $j \in [k]$ we denote $\tilde{\bu}_j = \tilde{v}_j \tilde{\bw}_j$. We show that $\tilde{\bu}_1,\ldots,\tilde{\bu}_k$ satisfy the KKT conditions of the aforementioned problem. Since the objective there is convex and the constrains are affine, then it is a global optimum. Finally, we show that it implies global optimality of $\tilde{\btheta}$.
\end{proof}

\begin{remark}[Implications on margin maximization in predictor space for diagonal linear networks]
\label{rem:implications on predictor}
Theorems~\ref{thm:depth 2 linear negative} and~\ref{thm:depth 2 linear} imply analogous results on diagonal linear networks also in predictor space. 
%Recall that, as 
As
we discussed in Section~\ref{sec:preliminaries}, 
%margin maximization in parameter space induces implicit bias in predictor space. 
\cite{gunasekar2018bimplicit} showed that in depth-$2$ diagonal linear networks, minimizing $\norm{\btheta}_2$ under the constraints in Problem~\ref{eq:optimization problem} implies minimizing $\norm{\bbeta}_1$, where $\bbeta$ is the corresponding linear predictor.
Theorem~\ref{thm:depth 2 linear negative} can be easily extended to predictor space, namely, gradient flow on depth-$2$ linear diagonal networks might converge to a KKT point $\tilde{\btheta}$ of Problem~\ref{eq:optimization problem}, such that the corresponding linear predictor $\tilde{\bbeta}$ is not an optimum of the following problem:
	\begin{equation}
	\label{eq:diagonal predictor problem}
		\argmin_{\bbeta} \norm{\bbeta}_{1}  \;\;\;\; \text{s.t. } \;\;\; \forall i \in [n] \;\; y_i \inner{\bbeta,\bx_i} \geq 1~.
	\end{equation}
Moreover, by combining part (1) of Theorem~\ref{thm:depth 2 linear} with the result from \cite{gunasekar2018bimplicit}, we deduce that if gradient flow on a depth-$2$ diagonal linear network converges in direction to a KKT point $\tilde{\btheta}$ of Problem~\ref{eq:optimization problem} with non-zero weights vectors, then the corresponding linear predictor is a global optimum of Problem~\ref{eq:diagonal predictor problem}. 
%See Appendix~\ref{TODO} for a formal statement.
\end{remark}

We argue that since in practice gradient descent often converges to networks without zero-weight neurons, then part (1) of Theorem~\ref{thm:depth 2 linear} gives a useful positive result for depth-$2$ linear networks. However,
by part (2) of Theorem~\ref{thm:depth 2 linear}, 
%assuming that gradient flow converges to a network without zero neurons 
this assumption
is not sufficient for obtaining a positive result in the case of ReLU networks. Hence, we now consider a stronger assumption, namely, that the KKT point $\tilde{\btheta}$ is such that for every $\bx_i$ in the dataset the inputs to all hidden neurons in the computation $\Phi(\tilde{\btheta}; \bx_i)$ are non-zero. 
In the following theorem we show that in depth-$2$ ReLU networks, if 
the KKT point 
%is such that %the inputs to all hidden neurons are non-zero 
satisfies this condition
then it is guaranteed to be a local optimum of Problem~\ref{eq:optimization problem}. However, even under this condition it is not necessarily a global optimum. The proof is given in Appendix~\ref{app:proof of depth 2 relu positive} and uses ideas from the previous proofs, with some required modifications. 

\begin{theorem}
\label{thm:depth 2 relu positive}
	Let $\Phi$ be a depth-$2$ ReLU network in $\cn_{\text{no-share}}$
	% and $k$ hidden neurons 
	parameterized by $\btheta$. 
	%Namely, for $W \in \reals^{k \times d}$, $\bv \in \reals^k$ and $\btheta = [W,\bv]$ we have for every $\bx \in \reals^d$ that $\Phi(\btheta; \bx) = \bv^\top \sigma(W \bx) = \sum_{j \in [k]} v_j \sigma(\bw_j^\top \bx)$. 
	Consider minimizing either the exponential or the logistic loss over a dataset $\{(\bx_i,y_i)\}_{i=1}^n$ using gradient flow. Assume that there exists time $t_0$ such that $\cl(\btheta(t_0))<1$, and let $\tilde{\btheta}$ be the KKT point of Problem~\ref{eq:optimization problem} such that $\btheta(t)$ converges to $\tilde{\btheta}$ in direction (such $\tilde{\btheta}$ exists by Theorem~\ref{thm:known KKT}). 
	%Assume that for every $i \in [n]$ and $j \in [k]$ we have $\inner{\tilde{\bw}_j,\bx_i^j} \neq 0$. 
	Assume that for every $i \in [n]$ the inputs to all hidden neurons in the computation $\Phi(\tilde{\btheta}; \bx_i)$ are non-zero. 
	Then, $\tilde{\btheta}$ is a local optimum of Problem~\ref{eq:optimization problem}. However, it may not be a global optimum, even if the network $\Phi$ is fully connected. 
\end{theorem}

Note that in all the above theorems we do not allow shared weights. We now consider the case of depth-$2$ linear or ReLU networks in $\cn$, where the first layer is convolutional with disjoint patches (and hence has shared weights), and show that gradient flow does not always converge in direction to a local optimum, even when the inputs to all hidden neurons are non-zero (and hence there are no zero weights vectors). 
%We prove the theorem in Appendix~\ref{app:proof of depth 2 cnn}.

\begin{theorem}
\label{thm:depth 2 cnn}
	Let $\Phi$ be a depth-$2$ linear or ReLU network in $\cn$, parameterized by $\btheta = [\bw,\bv]$ for $\bw,\bv \in \reals^2$, such that for $\bx \in \reals^4$ we have $\Phi(\btheta; \bx) = \sum_{j=1}^2 v_j \sigma(\bw^\top \bx^{(j)})$ where $\bx^{(1)} = (x_1,x_2)$ and $\bx^{(2)} = (x_3,x_4)$. 
	Thus, $\Phi$ is a convolutional network with two disjoint patches.
	Consider minimizing either the exponential or the logistic loss using gradient flow. 
	Then, there exists a dataset $\{(\bx,y)\}$ of size $1$, and an initialization $\btheta(0)$, such that gradient flow converges to zero loss, and converges in direction to a KKT point $\tilde{\btheta} = [\tilde{\bw},\tilde{\bv}]$ of Problem~\ref{eq:optimization problem} which is not a local optimum. Moreover, we have $\inner{\tilde{\bw},\bx^{(j)}} \neq 0$ for $j \in \{1,2\}$. 
\end{theorem}
\begin{proof}[Proof idea (for the complete proof see Appendix~\ref{app:proof of depth 2 cnn})]
	Let $\bx = \left(4,\frac{1}{\sqrt{2}},-4, \frac{1}{\sqrt{2}} \right)^\top$ and $y = 1$. Let $\btheta(0)$ such that $\bw(0) = (0,1)^\top$ and $\bv(0) = 
\left( \frac{1}{\sqrt{2}}, \frac{1}{\sqrt{2}} \right)^\top$. Since $\bx^{(1)}$ and $\bx^{(2)}$ are symmetric w.r.t. $\bw(0)$, and $\bv(0)$ does not break this symmetry, then $\bw$ keeps its direction throughout the training. Thus, we show that gradient flow converges in direction to a KKT point $\tilde{\btheta}$ where $\tilde{\bw} = (0,1)^\top$ and $\tilde{\bv} = \left( \frac{1}{\sqrt{2}}, \frac{1}{\sqrt{2}} \right)^\top$. Then, we show that it is not a local optimum, since for every small $\epsilon>0$ the parameters $\btheta' = [\bw',\bv']$ with $\bw' = (\sqrt{\epsilon},1-\epsilon)^\top$ and $\bv' = \left( \frac{1}{\sqrt{2}} + \frac{\sqrt{\epsilon}}{2}, \frac{1}{\sqrt{2}} - \frac{\sqrt{\epsilon}}{2} \right)^\top$ satisfy the constraints of Problem~\ref{eq:optimization problem}, and we have $\snorm{\btheta'}<\snorm{\tilde{\btheta}}$.
\end{proof}

\section{Deep networks in $\cn$}
\label{sec:deep}

In this section we study the more general case of depth-$m$ neural networks in $\cn$, where $m \geq 2$.
First,
we show that for networks of depth at least $3$ in $\cn_{\text{no-share}}$, gradient flow may not converge to a local optimum of Problem~\ref{eq:optimization problem}, for both linear and ReLU networks, and even where there are no zero weights vectors and the inputs to all hidden neurons are non-zero. 
More precisely, we prove this claim for diagonal networks.

\begin{theorem}
\label{thm:deep negative}
	Let $m \geq 3$. Let $\Phi$ be a depth-$m$ linear or ReLU diagonal neural network parameterized by $\btheta$. 
	%parameterized by $\btheta = [\bw_1,\ldots,\bw_m]$.
	Consider minimizing either the exponential or the logistic loss using gradient flow. There exists a dataset $\{(\bx,y)\} \subseteq \reals^2 \times \{-1,1\}$ of size $1$ and an initialization $\btheta(0)$, such that gradient flow converges to zero loss, and converges in direction to a KKT point $\tilde{\btheta}$ of Problem~\ref{eq:optimization problem} which is not a local optimum. 
	Moreover, 
	%$\Phi(\tilde{\btheta};\cdot)$ is such that for the input $x$ 
	all inputs to neurons in the computation $\Phi(\tilde{\btheta};\bx)$ are non-zero.
\end{theorem}
\begin{proof}[Proof idea (for the complete proof see Appendix~\ref{app:proof of deep negative})]
Let $\bx=(1,1)^\top$ and $y=1$. Consider the initialization $\btheta(0)$ where $\bw_j(0) = (1,1)^\top$ for every $j \in [m]$. We show that gradient flow converges in direction to a KKT point $\tilde{\btheta} = [\tilde{\bw}_1,\ldots,\tilde{\bw}_m]$ such that $\tilde{\bw}_j = \left(2^{-1/m},2^{-1/m}\right)^\top$ for all $j \in [m]$. Then, we consider the parameters $\btheta' = [\bw'_1,\ldots,\bw'_m]$ such that for every $j \in [m]$ we have $\bw'_j = \left(\left(\frac{1+\epsilon}{2}\right)^{1/m}, \left(\frac{1-\epsilon}{2}\right)^{1/m} \right)^\top$, and show that if $\epsilon>0$ is sufficiently small, then $\btheta'$ satisfies the constraints in Problem~\ref{eq:optimization problem} and we have $\snorm{\btheta'} < \snorm{\tilde{\btheta}}$.
\end{proof}

%Thus, in linear networks with sparse weights of depth at least $3$ gradient flow may not converge in direction to a local optimum of Problem~\ref{eq:optimization problem} even when . 
%This 
Note that in the case of linear networks, the above result 
is in contrast to 
%linear 
networks with sparse weights of depth $2$ that converge to a global optimum by Theorem~\ref{thm:depth 2 linear}, and to fully-connected 
%linear 
networks of any depth that converge to a global optimum by Theorem~\ref{thm:deep positive linear}.
In the case of ReLU networks,
the above result 
%also implies that in ReLU networks with sparse weights of depth at least $3$ gradient flow does not converge in direction to a local optimum of Problem~\ref{eq:optimization problem}, even when all inputs to neurons in the computation $\Phi(\tilde{\btheta};\bx)$ are non-zero for every $\bx$ in the dataset. This 
is in contrast to the case of depth-$2$ 
%ReLU 
networks studied in Theorem~\ref{thm:depth 2 relu positive}, where it is guaranteed to converge to a local optimum.

In light of our negative results, we now consider a weaker notion of margin maximization, namely, maximizing the margin for each layer separately. 
Let $\Phi$ be a neural network of depth $m$ in $\cn$, parameterized by  $\btheta =[\bu^{(l)}]_{l=1}^m$. The maximum margin problem for a layer $l_0 \in [m]$ w.r.t. $\btheta_0 =[\bu^{(l)}_0]_{l=1}^m$ is the following:
\begin{equation}
\label{eq:optimization problem one layer}
	\min_{\bu^{(l_0)}} \frac{1}{2} \norm{\bu^{(l_0)}}^2 \;\;\;\; \text{s.t. } \;\;\; \forall i \in [n] \;\; y_i \Phi(\btheta'; \bx_i) \geq 1~,
\end{equation}
where $\btheta' =[\bu^{(1)}_0,\ldots,\bu^{(l_0-1)}_0,\bu^{(l_0)},\bu^{(l_0+1)}_0,\ldots,\bu^{(m)}_0]$. We show a positive result for linear networks:

\begin{theorem}
\label{thm:positive each layer linear}
	Let $m \geq 2$. Let $\Phi$ be any depth-$m$ linear neural network in $\cn$, parameterized by $\btheta =[\bu^{(l)}]_{l=1}^m$.
	Consider minimizing either the exponential or the logistic loss over a dataset $ \{(\bx_i,y_i)\}_{i=1}^n$ using gradient flow. Assume that there exists time $t_0$ such that $\cl(\btheta(t_0))<1$.  Then, gradient flow converges in direction to a KKT point $\tilde{\btheta}=[\tilde{\bu}^{(l)}]_{l=1}^m$ of Problem~\ref{eq:optimization problem}, such that for every layer $l \in [m]$ the parameters vector $\tilde{\bu}^{(l)}$ is a global optimum of Problem~\ref{eq:optimization problem one layer} w.r.t. $\tilde{\btheta}$.
\end{theorem}

The theorem follows by noticing that if $\Phi$ is a linear network, then the constraints in Problem~\ref{eq:optimization problem one layer} are affine, and its KKT conditions are implied by the KKT conditions of Problem~\ref{eq:optimization problem}.
See Appendix~\ref{app:proof of positive each layer linear} for the formal proof.
Note that by Theorems~\ref{thm:depth 2 linear negative}, \ref{thm:depth 2 cnn} and~\ref{thm:deep negative}, linear networks in $\cn$ might converge in direction to a KKT point $\tilde{\btheta}$, which is not a local optimum of Problem~\ref{eq:optimization problem}. However, Theorem~\ref{thm:positive each layer linear} implies that each layer in $\tilde{\btheta}$ is a global optimum of Problem~\ref{eq:optimization problem one layer}. Hence, any improvement to $\tilde{\btheta}$ requires changing at least two layers simultaneously. 

While in linear networks gradient flow maximize the margin for each layer separately,
in the following theorem (which we prove in Appendix~\ref{app:proof of negative each layer relu}) we show that this claim does not hold for ReLU networks: Already for fully-connected networks of depth-$2$ gradient flow may not converge in direction to a local optimum of Problem~\ref{eq:optimization problem one layer}.

\begin{theorem}
\label{thm:negative each layer relu}
	Let $\Phi$ be a fully-connected depth-$2$ ReLU network with input dimension $2$ and $4$ hidden neurons parameterized by $\btheta$. 
	%Namely, for $\btheta = [\bw_1,\bw_2,v_1,v_2]$ and $\bx \in \reals^2$ we have $\Phi(\btheta; \bx) = \sum_{l=1}^2 v_l \sigma(\bw_l^\top \bx)$. 
	Consider minimizing either the exponential or the logistic loss using gradient flow. 
	There exists a dataset 
	%$\{(\bx_1,y_1),(\bx_2,y_2),(\bx_3,y_3),(\bx_4,y_4)\}$ 
	and an initialization $\btheta(0)$ such that gradient flow converges to zero loss, and converges in direction to a KKT point $\tilde{\btheta}$ of Problem~\ref{eq:optimization problem}, such that the weights of the first layer are not a local optimum of Problem~\ref{eq:optimization problem one layer} w.r.t. $\tilde{\btheta}$.
\end{theorem}

Finally, we show that in ReLU networks in $\cn$ of any depth, if the KKT point to which gradient flow converges in direction is such that the inputs to hidden neurons are non-zero, then it must be a local optimum of Problem~\ref{eq:optimization problem one layer} (but not necessarily a global optimum).
The proof follows the ideas from the proof of Theorem~\ref{thm:positive each layer linear}, with some required modifications, and is given in Appendix~\ref{app:proof of positive each layer relu}. 
 
\begin{theorem}
\label{thm:positive each layer relu}
	Let $m \geq 2$. Let $\Phi$ be any depth-$m$ ReLU network in $\cn$ parameterized by $\btheta =[\bu^{(l)}]_{l=1}^m$.
	Consider minimizing either the exponential or the logistic loss over a dataset $\{(\bx_i,y_i)\}_{i=1}^n$ using gradient flow, and assume that there exists time $t_0$ such that $\cl(\btheta(t_0))<1$. 
	Let $\tilde{\btheta}=[\tilde{\bu}^{(l)}]_{l=1}^m$ be the KKT point of Problem~\ref{eq:optimization problem} such that $\btheta(t)$ converges to $\tilde{\btheta}$ in direction (such $\tilde{\btheta}$ exists by Theorem~\ref{thm:known KKT}).
	Let $l \in [m]$ and assume that for every $i \in [n]$ the inputs to all neurons in layers $\geq l$ in the computation $\Phi(\tilde{\btheta};
	\bx_i)$ are non-zero. Then, the parameters vector $\tilde{\bu}^{(l)}$ is a local optimum of Problem~\ref{eq:optimization problem one layer} w.r.t. $\tilde{\btheta}$. However, it may not be a global optimum.
\end{theorem}

%stam

%stam

\subsection*{Acknowledgements}
This research is supported in part by European Research Council (ERC) grant 754705.

\bibliographystyle{abbrvnat}
\bibliography{bib}

\appendix

\section{Additional related work}
\label{app:more related}

%The implicit bias of neural networks in classification tasks was widely studied in recent years.
 \cite{soudry2018implicit} showed that gradient descent on linearly-separable binary classification problems with exponentially-tailed losses (e.g., the exponential loss and the logistic loss), converges to the maximum $\ell_2$-margin direction. This analysis was extended to other loss functions, tighter convergence rates, non-separable data, and variants of gradient-based optimization algorithms \citep{nacson2019convergence,ji2018risk,ji2020gradient,gunasekar2018characterizing,shamir2020gradient,ji2021characterizing}.

As detailed in Section~\ref{sec:preliminaries}, \cite{lyu2019gradient} and \cite{ji2020directional} showed that gradient flow on homogeneous neural networks with exponential-type losses converge in direction to a KKT point of the maximum-margin problem in parameter space. 
%Similar results under stronger assumptions were previously obtained in \cite{nacson2019lexicographic,gunasekar2018bimplicit}. 
The implications of margin maximization in parameter space on the implicit bias in predictor space for linear neural networks were studied in \cite{gunasekar2018bimplicit} (as detailed in Section~\ref{sec:preliminaries}) and also in \cite{jagadeesan2021inductive,ergen2021convex,ergen2021revealing}.
Moreover, several recent works considered implications of convergence to a KKT point of the maximum-margin problem, without assuming that the KKT point is optimal: \cite{safran2022effective} proved a generalization bound in univariate depth-$2$ ReLU networks, \cite{vardi2022gradient} proved bias towards non-robust solutions in depth-$2$ ReLU networks, and \cite{haim2022reconstructing} showed that training data can be reconstructed from trained networks.
Margin maximization in predictor space for fully-connected linear networks was shown by \cite{ji2020directional} (as detailed in Section~\ref{sec:preliminaries}), and similar results under stronger assumptions were previously established in \cite{gunasekar2018bimplicit} and in \cite{ji2018gradient}.
The implicit bias in predictor space of diagonal and convolutional linear networks was studied in \cite{gunasekar2018bimplicit,moroshko2020implicit,yun2020unifying}. 
%\cite{lyu2021gradient} proved margin maximization in depth-$2$ Leaky-ReLU networks trained with the logistic loss on linearly separable and symmetric data. 
%The implicit bias in infinitely-wide two-layer homogeneous neural networks was studied in \cite{chizat2020implicit}. 
\citet{chizat2020implicit} studied the dynamics of gradient flow on infinite-width homogeneous two-layer networks with exponentially-tailed losses, and showed bias towards margin maximization w.r.t. a certain function norm known as the variation norm.
\cite{sarussi2021towards} studied gradient flow on two-layer leaky-ReLU networks, where the training data is linearly separable, and showed convergence to a linear classifier based on a certain assumption called \emph{Neural Agreement Regime (NAR)}.
\cite{phuong2020inductive} studied the implicit bias in depth-$2$ ReLU networks trained on \emph{orthogonally-separable} data.

\cite{lyu2021gradient} 
%proved margin maximization in depth-$2$ Leaky-ReLU networks trained with the logistic loss on linearly separable and symmetric data. 
studied the implicit bias in two-layer leaky-ReLU networks trained on linearly separable and symmetric data, and showed that gradient flow converges to a linear classifier which maximizes the $\ell_2$ margin.
They also gave constructions where a KKT point is not a global max-margin solution.
We note that their constructions do not imply any of our results. In particular, for the ReLU activation they showed a construction where there exists a KKT point which is not a global optimum of the max-margin problem, however this KKT point is not reachable with gradient flow. Thus, there does not exist an initialization such that gradient flow actually converges to this point. In our construction (Theorem~\ref{thm:depth 2 relu negative}) gradient flow converges to a suboptimal KKT point with constant probability over the initialization. Moreover, even their construction for the leaky-ReLU activation (which is their main focus) considers only global suboptimality, while we show local suboptimality for the ReLU activation.

Finally, the implicit bias of neural networks in regression tasks w.r.t. the square loss was also extensively studied in recent years (e.g., \cite{gunasekar2018implicit,razin2020implicit,arora2019implicit,belabbas2020implicit,eftekhari2020implicit,li2018algorithmic,ma2018implicit,woodworth2020kernel,gidel2019implicit,li2020towards,yun2020unifying,vardi2021implicit,azulay2021implicit,timor2022implicit}). This setting, however, is less relevant to our work.

For a broader discussion on the implicit bias in neural networks, in both classification and regression tasks, see a survey in \cite{vardi2022implicit}.

\section{Preliminaries on the KKT conditions}
\label{app:KKT}

Below we review the definition of the KKT condition for non-smooth optimization problems (cf. \cite{lyu2019gradient,dutta2013approximate}).

Let $f: \reals^d \to \reals$ be a locally Lipschitz function. The Clarke subdifferential \citep{clarke2008nonsmooth} at $\bx \in \reals^d$ is the convex set
\[
	\partial^\circ f(\bx) := \text{conv} \left\{ \lim_{i \to \infty} \nabla f(\bx_i) \; \middle| \; \lim_{i \to \infty} \bx_i = \bx,\; f \text{ is differentiable at } \bx_i  \right\}~.
\]
If $f$ is continuously differentiable at $\bx$ then $\partial^\circ f(\bx) = \{\nabla f(\bx) \}$.

Consider the following optimization problem
\begin{equation}
\label{eq:KKT nonsmooth def}
	\min f(\bx) \;\;\;\; \text{s.t. } \;\;\; \forall n \in [N] \;\; g_n(\bx) \leq 0~,
\end{equation}
where $f,g_1,\ldots,g_n : \reals^d \to \reals$ are locally Lipschitz functions. We say that $\bx \in \reals^d$ is a feasible point of Problem~\ref{eq:KKT nonsmooth def} if $\bx$ satisfies $g_n(\bx) \leq 0$ for all $n \in [N]$. We say that a feasible point $\bx$ is a KKT point if there exists $\lambda_1,\ldots,\lambda_N \geq 0$ such that 
\begin{enumerate}
	\item $\zero \in \partial^\circ f(\bx) + \sum_{n \in [N]} \lambda_n \partial^\circ g_n(\bx)$;
	\item For all $n \in [N]$ we have $\lambda_n g_n(\bx) = 0$.
\end{enumerate}

\section{Proofs}
\label{app:proofs}

\subsection{Auxiliary lemmas}
\label{sec:auxiliary}

Throughout our proofs we use the following two lemmas from \cite{du2018algorithmic}:

\begin{lemma}[\cite{du2018algorithmic}]
\label{lem:from du fully connected}
	Let $m \geq 2$, and let $\Phi$ be a depth-$m$  fully-connected linear or ReLU network parameterized by $\btheta = [W_1,\ldots,W_m]$. Suppose that for every $j \in [m]$ we have $W_j \in \reals^{d_j \times d_{j-1}}$. 
	Consider minimizing any differentiable loss function (e.g., the exponential or the logistic loss) over a dataset using gradient flow.
	%Consider gradient flow on a fully-connected depth-$m$ linear or ReLU network w.r.t. any differentiable loss function (e.g., the exponential or the logistic loss). $\btheta = [W_1,\ldots,W_m]$ be the weight matrices where for every $j \in [m]$ we have $W_j \in \reals^{k_j \times k_{j-1}}$, $k_0=d$,$k_m=1$. 
	Then, for every $j \in [m-1]$ at all time $t$ we have 
\[
	\frac{d}{dt} \left(\norm{W_j}_F^2 - \norm{W_{j+1}}_F^2 \right) = 0~.
\]
Moreover, for every $j \in [m-1]$ and $i \in [d_j]$ we have
\[
	\frac{d}{dt} \left(\norm{W_j[i,:]}^2 - \norm{W_{j+1}[:,i]}^2 \right) = 0~,
\]
where $W_j[i,:]$ is the vector of incoming weights to the $i$-th neuron in the $j$-th hidden layer (i.e., the $i$-th row of $W_j$), and $W_{j+1}[:,i]$ is the vector of outgoing weights from this neuron (i.e., the $i$-th column of $W_{j+1}$).
\end{lemma}

\begin{lemma}[\cite{du2018algorithmic}]
\label{lem:from du sparse and cnn}
	Let $m \geq 2$, and let $\Phi$ be a depth-$m$ linear or ReLU network in $\cn$, parameterized by $\btheta = [\bu^{(1)},\ldots,\bu^{(m)}]$.
	Consider minimizing any differentiable loss function (e.g., the exponential or the logistic loss) over a dataset using gradient flow.
	Then, for every $j \in [m-1]$ at all time $t$ we have
	\[
		\frac{d}{dt} \left(\norm{\bu^{(j)}}^2 - \norm{\bu^{(j+1)}}^2 \right) = 0~.
	\]
\end{lemma}

Note that Lemma~\ref{lem:from du sparse and cnn} considers a larger family of neural networks since it allows sparse and shared weights, but Lemma~\ref{lem:from du fully connected} gives a stronger guarantee, since it implies balancedness between the incoming and outgoing weights of each hidden neuron separately.
In our proofs we will also need to use a balancedness property for each hidden neuron separately in depth-$2$ networks with sparse weights. Since this property is not implied by the above lemmas from \cite{du2018algorithmic}, we now prove it.

Before stating the lemma, let us introduce some required notations. Let $\Phi$ be a depth-$2$ network in $\cn_{\text{no-share}}$. We can always assume w.l.o.g. that the second layer is fully connected, namely, all hidden neurons are connected to the output neuron. Indeed, otherwise we can ignore the neurons that are not connected to the output neuron. For the network $\Phi$ we use the parameterization $\btheta = [\bw_1,\ldots,\bw_k,\bv]$, where $k$ is the number of hidden neurons. For every $j \in [k]$ the vector $\bw_j \in \reals^{p_j}$ is the weights vector of the $j$-th hidden neuron, and we have $1 \leq p_j \leq d$ where $d$ is the input dimention. For an input $\bx \in \reals^d$ we denote by $\bx^j \in \reals^{p_j}$ a sub-vector of $\bx$, such that $\bx^j$ includes the coordinates of $\bx$ that are connected to the $j$-th hidden neuron. Thus, given $\bx$, the input to the $j$-th hidden neuron is $\inner{\bw_j,\bx^j}$. The vector $\bv \in \reals^k$ is the weights vector of the second layer. Overall, we have $\Phi(\btheta; \bx) = \sum_{j \in [k]} v_j \sigma(\bw_j^\top \bx^j)$.

\begin{lemma}
\label{lem:extending du}
	Let $\Phi$ be a depth-$2$ linear or ReLU network in $\cn_{\text{no-share}}$, parameterized by $\btheta = [\bw_1,\ldots,\bw_k,\bv]$. 
	Consider minimizing any differentiable loss function (e.g., the exponential or the logistic loss) over a dataset using gradient flow.
	Then, for every $j \in [k]$ at all time $t$ we have
	\[
		\frac{d}{dt} \left(\norm{\bw_j}^2 - v_j^2 \right) = 0~.
	\]
\end{lemma}
\begin{proof}
	We have
	\[
		\cl(\btheta)
		= \sum_{i \in [n]} \ell \left( y_i \Phi(\btheta; \bx_i)\right)
		=  \sum_{i \in [n]} \ell \left( y_i \sum_{l \in [k]}v_l \sigma(\bw_l^\top \bx_i^j) \right)~.
	\]
	Hence
	\begin{align*}
		\frac{d}{dt} \left(\norm{\bw_j}^2 \right)
		&= 2 \inner{\bw_j, \frac{d \bw_j}{dt} }
		= -2 \inner{\bw_j, \nabla_{\bw_j} \cl(\btheta)}
		\\
		&= -2  \sum_{i \in [n]} \ell' \left( y_i \sum_{l \in [k]}v_l \sigma(\bw_l^\top \bx_i^l) \right) \cdot y_i v_j \sigma'(\bw_j^\top \bx_i^j) \bw_j^\top \bx_i^j
		\\
		&= -2  \sum_{i \in [n]} \ell' \left( y_i \sum_{l \in [k]}v_l \sigma(\bw_l^\top \bx_i^l) \right) \cdot y_i v_j \sigma(\bw_j^\top \bx_i^j)~.
	\end{align*}
	Moreover,
	\begin{align*}
		\frac{d}{dt} \left(v_j^2 \right)
		&= 2 v_j \frac{d v_j}{dt} 
		= - 2 v_j \nabla_{v_j} \cl(\btheta)
		\\
		&= -2 v_j  \sum_{i \in [n]} \ell' \left( y_i \sum_{l \in [k]}v_l \sigma(\bw_l^\top \bx_i^l) \right) \cdot y_i  \sigma(\bw_j^\top \bx_i^j)~.
	\end{align*}
	Hence the lemma follows.
\end{proof}

Using the above lemma, we show the following:

\begin{lemma}
\label{lem:balanced tilde}
	Let $\Phi$ be a depth-$2$ linear or ReLU network in $\cn_{\text{no-share}}$, parameterized by $\btheta = [\bw_1,\ldots,\bw_k,\bv]$. 
	Consider minimizing any differentiable loss function (e.g., the exponential or the logistic loss) over a dataset using gradient flow starting from $\btheta(0)$.
	Assume that $\lim_{t \to \infty}\norm{\btheta(t)} = \infty$ and that $\btheta(t)$ converges in direction to $\tilde{\btheta} = [\tilde{\bw}_1,\ldots,\tilde{\bw}_k,\tilde{\bv}]$, i.e., $\tilde{\btheta} = \snorm{\tilde{\btheta}} \cdot \lim_{t \to \infty}\frac{\btheta(t)}{\norm{\btheta(t)}}$. 
	Then, for every $l \in [k]$ we have $\norm{\tilde{\bw}_l} = |\tilde{v}_l |$.
\end{lemma}
\begin{proof}
	For every $l \in [k]$, let $\Delta_l = \norm{\bw_l(0)}^2 - v_l(0)^2$.
By Lemma~\ref{lem:extending du}, we have for every $l \in [k]$ and $t \geq 0$ that  $\norm{\bw_l(t)}^2 - v_l(t)^2 = \Delta_l$, namely, the differences between the square norms of the incoming and outgoing weights of each hidden neuron remain constant during the training. 
Hence, we have
\[
	|\tilde{v}_l| 
	= \snorm{\tilde{\btheta}} \cdot \lim_{t \to \infty} \frac{| v_l(t) |}{\norm{\btheta(t)}}
	=  \snorm{\tilde{\btheta}} \cdot \lim_{t \to \infty} \frac{\sqrt{ \norm{\bw_l(t)}^2 - \Delta_l }}{\norm{\btheta(t)}}~.
\]
Thus, if $\lim_{t \to \infty}\norm{\bw_l(t)} = \infty$, then we have $| \tilde{v}_l | =  \snorm{\tilde{\btheta}}  \cdot \lim_{t \to \infty} \frac{\norm{\bw_l(t)}}{\norm{\btheta(t)}}=\norm{\tilde{\bw}_l}$.

Assume now that $\norm{\bw_l(t)} \not \to \infty$. By the definition of $\tilde{\btheta}$ we have $\norm{\tilde{\bw}_l} = \snorm{\tilde{\btheta}}  \cdot \lim_{t \to \infty} \frac{\norm{\bw_l(t)}}{\norm{\btheta(t)}}$. Since $\lim_{t \to \infty} \frac{\norm{\bw_l(t)}}{\norm{\btheta(t)}}$ exists and $\lim_{t \to \infty}\norm{\btheta(t)} = \infty$, then we have $\lim_{t \to \infty} \frac{\norm{\bw_l(t)}}{\norm{\btheta(t)}} = 0$. Hence, $\lim_{t \to \infty} \frac{| v_l(t) |}{\norm{\btheta(t)}} =  \lim_{t \to \infty} \frac{\sqrt{ \norm{\bw_l(t)}^2 - \Delta_l}}{\norm{\btheta(t)}}=0$. 
Therefore  $\norm{\tilde{\bw}_l} = \tilde{v}_l = 0$.
\end{proof}

\subsection{Proof of Theorem~\ref{thm:deep positive linear}}
\label{app:proof of deep positive linear}

Suppose that the network $\Phi$ is parameterized by $\btheta = [W^{(1)},\ldots,W^{(m)}]$. 
By Theorem~\ref{thm:known KKT}, gradient flow converges in direction to a KKT point $\tilde{\btheta} = [\tilde{W}^{(1)},\ldots,\tilde{W}^{(m)}]$ of Problem~\ref{eq:optimization problem}.
For every $l \in [m]$ let $\Delta_l = \norm{W^{(l)}(0)}_F^2 - \norm{W^{(1)}(0)}_F^2$.
By Lemma~\ref{lem:from du fully connected}, we have for every $l \in [m]$ and $t \geq 0$ that 
\begin{align*}
	\norm{W^{(l)}(t)}_F^2 - \norm{W^{(1)}(t)}_F^2 
	&= \sum_{j=1}^{l-1} \norm{W^{(j+1)}(t)}_F^2 - \norm{W^{(j)}(t)}_F^2
	= \sum_{j=1}^{l-1} \norm{W^{(j+1)}(0)}_F^2 - \norm{W^{(j)}(0)}_F^2
	\\
	&= \norm{W^{(l)}(0)}_F^2 - \norm{W^{(1)}(0)}_F^2 
	= \Delta_l~.
\end{align*}
Hence, we have
\[
	\norm{\tilde{W}^{(l)}}_F 
	= \snorm{\tilde{\btheta}} \cdot \lim_{t \to \infty} \frac{\norm{W^{(l)}(t)}_F }{\norm{\btheta(t)}}
	=  \snorm{\tilde{\btheta}} \cdot \lim_{t \to \infty} \frac{\sqrt{\norm{W^{(1)}(t)}_F^2 + \Delta_l }}{\norm{\btheta(t)}}~.
\]
Since by Theorem~\ref{thm:known KKT} we have $\lim_{t \to \infty} \norm{\btheta(t)} = \infty$, then $\lim_{t \to \infty}\norm{W^{(1)}(t)}_F  = \infty$, and we have
\[
	\norm{\tilde{W}^{(l)}}_F 
	=  \snorm{\tilde{\btheta}} \cdot \lim_{t \to \infty} \frac{\norm{W^{(1)}(t)}_F}{\norm{\btheta(t)}}
	= \norm{\tilde{W}^{(1)}}_F
	:= C~.
\]

By \cite{ji2020directional} (Proposition~4.4), when gradient flow on a fully-connected linear network w.r.t. the exponential loss or the logistic loss converges to zero loss, then we have the following.
There are unit vectors $\bv_0,\ldots,\bv_m$ such that 
\begin{equation*}
\label{eq:linear align}
	\lim_{t \to \infty} \frac{W^{(l)}(t)}{\norm{W^{(l)}(t)}_F} = \bv_l \bv_{l-1}^\top
\end{equation*} 
for every $l \in [m]$. Moreover, we have $\bv_m = 1$, and $\bv_0 = \bu$ where 
\[
	\bu := \argmax_{\norm{\bu}=1} \min_{i \in [n]} y_i \bu^\top \bx_i
\] 
is the unique linear max margin predictor.
%By Eq.~\ref{eq:linear align} 

Note that we have
\begin{align*}
	\frac{\tilde{W}^{(l)}}{C}
	= \frac{\tilde{W}^{(l)}}{\norm{\tilde{W}^{(l)}}_F}
	= \frac{ \snorm{\tilde{\btheta}} \cdot \lim_{t \to \infty} \frac{W^{(l)}(t) }{\norm{\btheta(t)}}}{ \snorm{\tilde{\btheta}} \cdot \lim_{t \to \infty} \frac{\norm{W^{(l)}(t)}_F }{\norm{\btheta(t)}}}
	= \lim_{t \to \infty} \frac{W^{(l)}(t)}{\norm{W^{(l)}(t)}_F} 
	= \bv_l \bv_{l-1}^\top~.
\end{align*}
Thus, $\tilde{W}^{(l)} = C \bv_l \bv_{l-1}^\top$ for every $l \in [m]$.

Let $\tilde{\bu} = \tilde{W}^{(m)} \cdot \ldots \cdot  \tilde{W}^{(1)} = C^m \bu$. Since $\tilde{\btheta}$ is a KKT point of Problem~\ref{eq:optimization problem}, we have for every $l \in [m]$
\[
	\tilde{W}^{(l)} = \sum_{i \in [n]} \lambda_i y_i \frac{\partial \Phi(\tilde{\btheta}; \bx_i)}{\partial W^{(l)}}~,
\]
where $\lambda_i \geq 0$ for every $i$, and $\lambda_i=0$ if $y_i \Phi(\tilde{\btheta}; \bx_i) \neq 1$. Since $\tilde{W}^{(l)}$ are non-zero then there is $i \in [n]$ such that $1= y_i \Phi(\tilde{\btheta}; \bx_i) = y_i \tilde{\bu}^\top \bx_i = y_i C^m \bu^\top \bx_i$. 
Likewise, since $\tilde{\btheta}$ satisfies the constraints of Problem~\ref{eq:optimization problem}, then for every $i \in [n]$ we have $1 \leq y_i \Phi(\tilde{\btheta}; \bx_i) =  y_i C^m \bu^\top \bx_i$. 
Since, $\bu$ is a unit vector that maximized the margin, then we have 
\begin{equation}
\label{eq:tildeu minimal}
	\norm{\tilde{\bu}} = 
	C^m = \min \norm{\bu'} \text{ s.t. } y_i \bu'^\top \bx_i \geq 1 \text{ for all } i \in [n]~.
\end{equation}

Assume toward contradiction that there is $\btheta'$ with $\snorm{\btheta'} < \snorm{\tilde{\btheta}}$ that satisfies the constraints in Problem~\ref{eq:optimization problem}. Let $\bu' =  W'^{(m)} \cdot \ldots \cdot W'^{(1)}$. By Eq.~\ref{eq:tildeu minimal} we have $\snorm{\bu'} \geq \snorm{\tilde{\bu}} = C^m$. Moreover, we have $\norm{\bu'}  = \norm{W'^{(m)} \cdot \ldots \cdot W'^{(1)}} \leq \prod_{l \in [m]} \norm{W'^{(l)}}_F$ due to the submultiplicativity of the Frobenius norm. Hence $\prod_{l \in [m]} \norm{W'^{(l)}}_F \geq C^m$. The following lemma implies that 
\[
	\norm{\btheta'}^2 = \sum_{l \in [m]} \norm{W'^{(l)}}_F^2 \geq m \cdot C^2 =  \sum_{l \in [m]} \norm{\tilde{W}^{(l)}}_F^2 = \norm{\tilde{\btheta}}^2
\]
in contradiction to our assumption, and thus completes the proof.

\begin{lemma}
	Let $a_1,\ldots,a_m$ be real numbers such that $\prod_{j \in [m]} a_j \geq C^m$ for some $C \geq 0$. Then $\sum_{j \in [m]} a_j^2 \geq m \cdot C^2$.
\end{lemma}
\begin{proof}
	It suffices to prove the claim for the case where $\prod_{j \in [m]} a_j = C^m$. Indeed, if $\prod_{j \in [m]} a_j > C^m$ then we can replace some $a_j$ with an appropriate $a'_j$ such that $|a'_j| < |a_j|$ and we only decrease $\sum_{j \in [m]} a_j^2$.	
	Consider the following problem
	\[
		\min \frac{1}{2} \sum_{j \in [m]} a_j^2 \;\;\;\; \text{s.t. } \;\;\;\prod_{j \in [m]} a_j = C^m~.
	\]
	Using the Lagrange multipliers we obtain that there is some $\lambda \in \reals$ such that for every $l \in [m]$ we have $a_l = \lambda \cdot \prod_{j \neq l} a_j$. Thus, $a_l^2 = \lambda \cdot \prod_{j \in [m]} a_j$. It implies that $a_1^2 = \ldots = a_m^2$. Since $\prod_{j \in [m]} a_j = C^m$ then $|a_j|=C$ for every $j \in [m]$. Hence,  $\sum_{j \in [m]} a_j^2 = m C^2$.
\end{proof}

\subsection{Proof of Theorem~\ref{thm:depth 2 relu negative}}
\label{app:proof of depth 2 relu negative}

%Let $\bx_1 = \left(1,\frac{1}{4}\right)^\top$, $\bx_2 = \left(-1,\frac{1}{4}\right)^\top$, $y_1=y_2=1$. Let $S=\{(\bx_1,y_1),(\bx_2,y_2)\}$ be a dataset. 
%Consider the initialization $\bw_1(0) = (0,2)^\top$, $\bw_2(0) = (0,-2)^\top$, $v_1(0) = v_2(0) = 2$.
%Assume that the 
Consider an
initialization $\btheta(0)$ is such that $\bw_1(0)$ satisfies $\inner{\bw_1(0),\bx_1}>0$ and $\inner{\bw_1(0),\bx_2}>0$, and $\bw_2(0)$ satisfies $\inner{\bw_2(0),\bx_1}<0$ and $\inner{\bw_2(0),\bx_2}<0$. Moreover, assume that $v_1(0)>0$.
%Note that for $\btheta(0)$ drawn according to the Xavier initialization these requirements hold with positive probability. Indeed, let $\delta \in (0,\pi)$ be the angle between $\bx_1$ and $\bx_2$\footnote{$\delta$ equals approximately $0.844 \pi$.}. Since $\bw_1(0)$ is drawn from a spherically symmetric distribution then the probability that $\inner{\bw_1(0),\bx_1}>0$ and $\inner{\bw_1(0),\bx_2}>0$ is $\frac{\pi - \delta}{2\pi}>0$. Likewise, since $\bw_2(0)$ is also drawn independently from a spherically symmetric distribution then with probability $\frac{\pi - \delta}{2\pi}>0$ we have $\inner{\bw_2(0),\bx_1}<0$ and $\inner{\bw_2(0),\bx_2}<0$. Finally, we have $v_1(0)>0$ with probability $\frac{1}{2}$.

Note that for every $\btheta$ such that $\inner{\bw_2,\bx_1}<0$ and $\inner{\bw_2,\bx_2}<0$ we have
\begin{align*}
	\nabla_{\bw_2} \cl(\btheta) 
	&= \sum_{i=1}^2 \ell'(y_i \Phi(\btheta; \bx_i)) \cdot y_i \nabla_{\bw_2}\Phi(\btheta; \bx_i) 
	=  \sum_{i=1}^2 \ell'(y_i \Phi(\btheta; \bx_i)) \cdot y_i \nabla_{\bw_2} \left[ v_1 \sigma(\bw_1^\top \bx_i) + v_2 \sigma(\bw_2^\top \bx_i) \right]
	\\
	&=  \sum_{i=1}^2 \ell'(y_i \Phi(\btheta; \bx_i)) \cdot y_i  v_2 \sigma'(\bw_2^\top \bx_i) \bx_i
	= \zero~.
\end{align*}
and 
\begin{align*}
	\nabla_{v_2} \cl(\btheta) 
	&= \sum_{i=1}^2 \ell'(y_i \Phi(\btheta; \bx_i)) \cdot y_i \nabla_{v_2}\Phi(\btheta; \bx_i) 
	= \sum_{i=1}^2 \ell'(y_i \Phi(\btheta; \bx_i)) \cdot y_i \nabla_{v_2} \left[ v_1 \sigma(\bw_1^\top \bx_i) + v_2 \sigma(\bw_2^\top \bx_i) \right]
	\\
	&=  \sum_{i=1}^2 \ell'(y_i \Phi(\btheta; \bx_i)) \cdot y_i \sigma(\bw_2^\top \bx_i) 
	= 0~.
\end{align*}
Hence, $\bw_2$ and $v_2$ get stuck in their initial values.
Moreover, we have
\begin{align*}
	\nabla_{v_1} \cl(\btheta) 
	%&= \sum_{i=1}^2 \ell'(y_i \Phi(\btheta; \bx_i)) \cdot y_i \nabla_{v_1}\Phi(\btheta; \bx_i) 
	= \sum_{i=1}^2 \ell'(y_i \Phi(\btheta; \bx_i)) \cdot y_i \nabla_{v_1} \left[ v_1 \sigma(\bw_1^\top \bx_i) + v_2 \sigma(\bw_2^\top \bx_i) \right]
	%\\
	=  \sum_{i=1}^2 \ell'(y_i \Phi(\btheta; \bx_i)) \cdot \sigma(\bw_1^\top \bx_i) 
	\leq 0~.
\end{align*}
Therefore, for every $t \geq 0$ we have $v_1(t) \geq v_1(0) > 0$.

We denote $\bw_1 = (w_1[1],w_1[2])$. Since $\inner{\bw_1(0),\bx_j}>0$ for $j \in \{1,2\}$ then $w_1[2](0) > 0$. Assume w.l.o.g. that $w_1[1](0) \geq 0$ (the case where $w_1[1](0) \leq 0$ is similar).
For every $\bw_1$ that satisfies $w_1[2] \geq 0$ and $0 \leq w_1[1] \leq w_1[1](0)$ we have $\inner{\bw_1,\bx_1}>\inner{\bw_1,\bx_2}>0$. Thus,
\begin{align*}
	\nabla_{\bw_1} \cl(\btheta) 
	%&= \sum_{i=1}^2 \ell'(y_i \Phi(\btheta; \bx_i)) \cdot y_i \nabla_{\bw_1}\Phi(\btheta; \bx_i) 
	&=  \sum_{i=1}^2 \ell'(y_i \Phi(\btheta; \bx_i)) \cdot y_i \nabla_{\bw_1} \left[ v_1 \sigma(\bw_1^\top \bx_i) + v_2 \sigma(\bw_2^\top \bx_i) \right]
	\\
	&=  \sum_{i=1}^2 \ell'(y_i (v_1 \sigma(\bw_1^\top \bx_i) + 0)) \cdot y_i  v_1 \sigma'(\bw_1^\top \bx_i) \bx_i
	\\
	&= \sum_{i=1}^2 \ell'(v_1 \bw_1^\top \bx_i) \cdot  v_1 \bx_i~.
\end{align*}
Since $\ell'$ is negative and monotonically increasing, and since $v_1 \bw_1^\top \bx_1 > v_1 \bw_1^\top \bx_2$, then $\frac{d w_1[1]}{dt} \leq 0$. 
Also, $\frac{d w_1[2]}{dt} > 0$. 
Moreover, if $w_1[1]=0$ then $v_1 \bw_1^\top \bx_1 = v_1 \bw_1^\top \bx_2$ and thus $\frac{d w_1[1]}{dt} = 0$.
Hence, for every $t$ we have $w_1[2](t) \geq w_1[2](0) > 0$ and $0 \leq w_1[1](t) \leq w_1[1](0)$.

If $\cl(\btheta) \geq 1$ then for some $i \in \{1,2\}$ we have $\ell(y_i \Phi(\btheta;\bx_i)) \geq \frac{1}{2}$ and hence $\ell'(y_i \Phi(\btheta;\bx_i)) \leq c$ for some constant $c<0$. Since we also have $v_1 \geq v_1(0) > 0$, we have
\[
	\frac{d w_1[2]}{dt} \geq -c \cdot v_1(0) \cdot \frac{1}{4}~.
\]
Therefore, if the initialization $\btheta(0)$ is such that $\cl(\btheta) \geq 1$ then $w_1[2](t)$ increases at rate at least $\frac{(-c) \cdot v_1(0)}{4}$ while $w_1[1](t)$ remains in $[0,w_1[1](0)]$. Note that for such $w_1[1]$  and $v_1 \geq v_1(0) > 0$, if $w_1[2]$ is sufficiently large then we have $v_1 \inner{\bw_1,\bx_i} \geq 1$ for $i \in \{1,2\}$.  Hence, there is some $t_0$ such that  $\cl(\btheta(t_0)) \leq 2 \ell(1) < 1$ for both the exponential loss and the logistic loss.

%Note that $\cl(\btheta(0)) = 2 \ell(1) < 1$ for both the exponential loss and the logistic loss, and therefore 
Therefore,
by Theorem~\ref{thm:known KKT} gradient flow converges in direction to a KKT point of Problem~\ref{eq:optimization problem}, and we have $
\lim_{t \to \infty}\cl(\btheta(t))=0$ and $\lim_{t \to \infty}\norm{\btheta(t)}=\infty$. It remains to show that it does not converge in direction to a local optimum of Problem~\ref{eq:optimization problem}.

Let $\bar{\btheta} = \lim_{t \to \infty} \frac{\btheta(t)}{\norm{\btheta(t)}}$. We denote $\bar{\btheta} = [\bar{\bw}_1,\bar{\bw}_2,\bar{v}_1,\bar{v}_2]$.
We show that $\bar{\bw}_1  = \frac{1}{\sqrt{2}} (0,1)^\top$, $\bar{v}_1 = \frac{1}{\sqrt{2}}$, $\bar{\bw}_2 = \zero$ and $\bar{v}_2 = 0$. 
By Lemma~\ref{lem:from du fully connected}, we have for every $t \geq 0$ that  $v_1(t)^2 - \norm{\bw_1(t)}^2 = v_1(0)^2 - \norm{\bw_1(0)}^2 := \Delta$. Since for every $t$ we have $\bw_2(t) = \bw_2(0)$ and $v_2(t)=v_2(0)$, and since $\lim_{t \to \infty}\norm{\btheta(t)} = \infty$ then we have $\lim_{t \to \infty}\norm{\bw_1(t)}=\infty$ and $\lim_{t \to \infty} | v_1(t) |=\infty$. 
Also, since  $\lim_{t \to \infty}\norm{\bw_1(t)}=\infty$ and $w_1[1](t) \in [0,w_1[1](0)]$ then $\lim_{t \to \infty} w_1[2](t) = \infty$.
Note that
\[
	\norm{\btheta(t)} 
	= \sqrt{ \norm{\bw_1(t)}^2 + v_1(t)^2 + \norm{\bw_2(0)}^2 + v_2(0)^2 }
	= \sqrt{ \Delta + 2\norm{\bw_1(t)}^2 + \norm{\bw_2(0)}^2 + v_2(0)^2 }~.
\]
Since $w_1[1](t) \in [0,w_1[1](0)]$ and $\norm{\btheta(t)} \to \infty$, we have
\[
	\bar{\bw}_1[1] 
	= \lim_{t \to \infty} \frac{w_1[1](t)}{\norm{\btheta(t)}}
	%= \lim_{t \to \infty} \frac{w_1[1](t)}{ \sqrt{ \Delta + 2\norm{\bw_1(t)}^2 + \norm{\bw_2(0)}^2 + v_2(0)^2 }}
	= 0~.
\]
Moreover,
\[
	\bar{\bw}_1[2] 
	= \lim_{t \to \infty} \frac{w_1[2](t)}{\norm{\btheta(t)}}
	= \lim_{t \to \infty} \sqrt{ \frac{(w_1[2](t))^2}{\Delta + 2 (w_1[1](t))^2 + 2 (w_1[2](t))^2 + \norm{\bw_2(0)}^2 + v_2(0)^2}}
	= \frac{1}{\sqrt{2}}~,
\]
and
\[
	\bar{\bw}_2 
	= \lim_{t \to \infty} \frac{\bw_2(t)}{\norm{\btheta(t)}}
	= \lim_{t \to \infty} \frac{\bw_2(0)}{\norm{\btheta(t)}}
	= \zero~.
\]
Finally, by Lemma~\ref{lem:balanced tilde} and since $v_1(t) > 0$, we have $\bar{v}_1 = \norm{\bar{\bw}_1} = \frac{1}{\sqrt{2}}$. By Lemma~\ref{lem:balanced tilde} we also have $| \bar{v}_2 | =  \norm{\bar{\bw}_2} = 0$.

Next, we show that $\bar{\btheta}$ does not point at the direction of a local optimum of Problem~\ref{eq:optimization problem}.
Let $\tilde{\btheta} = [\tilde{\bw}_1,\tilde{\bw}_2,\tilde{v}_1,\tilde{v}_2]$ be a KKT point of Problem~\ref{eq:optimization problem} that points at the direction of $\bar{\btheta}$. Such $\tilde{\btheta}$ exists since $\btheta(t)$ converges in direction to a KKT point. Thus, we have $\tilde{\bw}_2 = \zero$, $\tilde{v}_2 = 0$, $\tilde{\bw}_1 = \alpha (0,1)^\top$ and $\tilde{v}_1 = \alpha$ for some $\alpha>0$.
%Note that the vector $\btheta=(\bw_1,\bw_2,v_1,v_2)$ with the smallest norm such that $\btheta$ points in the direction of $\bar{\btheta}$ and for every $i \in \{1,2\}$ we have $y_i \Phi(\btheta; \bx_i) \geq 1$ satisfies $\bw_1 = (0,2)^\top$, $\bw_2=\zero$, $v_1 = 2$ and $v_2=0$.
Since $\tilde{\btheta}$ satisfies the KKT conditions, we have
\[
	\tilde{\bw}_1
	= \sum_{i=1}^2 \lambda_i \nabla_{\bw_1} \left( y_i \Phi(\tilde{\btheta}; \bx_i) \right)
	= \sum_{i=1}^2 \lambda_i y_i \left(\tilde{v}_1 \sigma'(\tilde{\bw}_1^\top \bx_i) \bx_i \right)~,
\]
where $\lambda_i \geq 0$ and $\lambda_i=0$ if $y_i \Phi(\tilde{\btheta};\bx_i) \neq 1$. 
Note that the KKT condition should be w.r.t. the Clarke subdifferential, but since $\tilde{\bw}_1^\top \bx_i > 0$ for $i \in \{1,2\}$ then we use here the gradient.
Hence, there is $i \in \{1,2\}$ such that $y_i \Phi(\tilde{\btheta};\bx_i) = 1$. Thus,
\[
	1 
	= y_i \Phi(\tilde{\btheta};\bx_i) 
	= \tilde{v}_1 \sigma(\tilde{\bw}_1^\top \bx_i) + \tilde{v}_2 \sigma(\tilde{\bw}_2^\top \bx_i)
	= \alpha \cdot \frac{\alpha}{4} + 0
	=  \frac{\alpha^2}{4}~.
\]
Therefore, $\alpha=2$ and we have $\tilde{\bw}_1 = (0,2)^\top$ and $\tilde{v}_1 = 2$.

In order to show that $\tilde{\btheta}$ is not a local optimum, we show that for every $0<\epsilon'<1$ there exists some $\btheta'$ such that $\norm{\btheta' - \tilde{\btheta}} \leq \epsilon'$, $\btheta'$ satisfies $\Phi(\btheta'; \bx_i) \geq 1$ for every $i \in \{1,2\}$, and $\snorm{\btheta'} < \snorm{\tilde{\btheta}}$.
Let $\epsilon = \frac{\epsilon'^2}{9} < \frac{1}{2}$.
Let $\btheta'= [\bw'_1,\bw'_2,v'_1,v'_2]$ be such that $\bw'_1 = (\frac{\epsilon}{2},2-2\epsilon)^\top$, $\bw'_2 = (-\sqrt{2\epsilon},0)^\top$, $v'_1 = 2$ and $v'_2 = \sqrt{2\epsilon}$. Note that 
\begin{align*}
	%q_1(\btheta') 
	\Phi(\btheta'; \bx_1)
	&= 2 \cdot \sigma\left( (\frac{\epsilon}{2},2-2\epsilon) (1,\frac{1}{4})^\top\right) + \sqrt{2\epsilon} \cdot \sigma\left( (-\sqrt{2\epsilon},0) (1,\frac{1}{4})^\top\right)
	\\
	&= 2 \cdot \sigma\left(  \frac{\epsilon}{2} + \frac{1}{2} - \frac{\epsilon}{2} \right) + \sqrt{2\epsilon} \cdot \sigma\left(-\sqrt{2\epsilon} \right)
	%= 2 \cdot \left( \frac{\epsilon}{2} + \frac{1}{2} - \frac{\epsilon}{2} \right) + 0
	= 1~,
\end{align*}
and
\begin{align*}
	%q_2(\btheta') 
	\Phi(\btheta'; \bx_2)
	&= 2 \cdot \sigma\left( (\frac{\epsilon}{2},2-2\epsilon) (-1,\frac{1}{4})^\top\right) + \sqrt{2\epsilon} \cdot \sigma\left( (-\sqrt{2\epsilon},0) (-1,\frac{1}{4})^\top\right)
	\\
	&= 2 \cdot \sigma \left( -\frac{\epsilon}{2} + \frac{1}{2} - \frac{\epsilon}{2} \right) +  \sqrt{2\epsilon} \cdot \sigma \left( \sqrt{2\epsilon} \right) 
	= 1-2\epsilon + 2\epsilon
	= 1~.
\end{align*}
We also have
\begin{align*}
	\norm{\btheta' - \tilde{\btheta}}^2
	 &= \norm{\bw'_1 - \tilde{\bw}_1}^2 + \norm{\bw'_2 - \tilde{\bw}_2}^2 + (v'_1 - \tilde{v}_1)^2 + (v'_2 - \tilde{v}_2)^2
	 \\
	 &= \left( \frac{\epsilon^2}{4} + 4 \epsilon^2 \right) + 2\epsilon + 0 + 2\epsilon
	 <9 \epsilon
	  = \epsilon'^2~.
\end{align*}
Finally, we have
\[
	\norm{\btheta'}^2 
	= \frac{\epsilon^2}{4} + 4 -8 \epsilon + 4 \epsilon^2 + 2 \epsilon + 4 + 2 \epsilon
	= 8 - 4 \epsilon + \frac{17 \epsilon^2}{4}
	< 8 - 4 \epsilon + \frac{17 \epsilon}{8}
	< 8 
	= \norm{\tilde{\btheta}}^2~.
\]
Thus, $\snorm{\btheta'} < \snorm{\tilde{\btheta}}$.

\subsection{Proof of Theorem~\ref{thm:depth 2 linear negative}}
\label{app:proof of depth 2 linear negative}

Let $\bx = (1,2)^\top$ and $y=1$. Let $\btheta(0)$ such that $\bw_1(0) = \bw_2(0) = (1,0)^\top$. 
Note that $\cl(\btheta(0)) = \ell(1) < 1$ for both linear and ReLU networks with the exponential loss or the logistic loss, and therefore by Theorem~\ref{thm:known KKT} gradient flow converges in direction to a KKT point $\tilde{\btheta}$ of Problem~\ref{eq:optimization problem}, and we have  $\lim_{t \to \infty}\cl(\btheta(t))=0$ and $\lim_{t \to \infty}\norm{\btheta(t)}=\infty$.
We denote $\bw_1 = (\bw_1[1],\bw_1[2])^\top$ and $\bw_2 = (\bw_2[1],\bw_2[2])^\top$.
Note that the initialization $\btheta(0)$ is such that the second hidden neuron has $0$ in both its incoming and outgoing weights. Hence, the gradient w.r.t. $\bw_1[2]$ and $\bw_2[2]$ is zero, and the second hidden neuron remains inactive during the training. Moreover, $\bw_1[1]$ and $\bw_2[1]$ are strictly increasing. Also, by Lemma~\ref{lem:extending du} we have for every $t \geq 0$ that $\bw_1[1](t)^2 = \bw_2[1](t)^2$. Overall, $\tilde{\btheta}$ is such that $\tilde{\bw}_1 = \tilde{\bw}_2 = (1,0)^\top$. Note that since the dataset is of size $1$, then every KKT point of Problem~\ref{eq:optimization problem} must label the input $\bx$ with exactly $1$.   

It remains to show that $\tilde{\btheta}$ is not local optimum. 
Let $0<\epsilon<1$, and let $\btheta' = [\bw'_1,\bw'_2]$ with $\bw'_1 = \bw'_2 = \left(\sqrt{1-\epsilon}, \sqrt{\frac{\epsilon}{2}}\right)^\top$. Note that $\btheta'$ satisfies the constraints of Problem~\ref{eq:optimization problem}, since $y \cdot \Phi(\btheta'; \bx) = 1-\epsilon + 2 \cdot \frac{\epsilon}{2} = 1$. Moreover, we have $\snorm{\tilde{\btheta}}^2 = 2$ and $\snorm{\btheta'}^2 = 2 \left(1-\epsilon + \frac{\epsilon}{2} \right) = 2-\epsilon$ and therefore $\snorm{\btheta'} < \snorm{\tilde{\btheta}}$.

\subsection{Proof of Theorem~\ref{thm:depth 2 linear}}
\label{app:proof of depth 2 linear}

\subsubsection{Proof of part 1}

We assume w.l.o.g. that the second layer is fully-connected, namely, all hidden neurons are connected to the output neuron, since otherwise we can ignore disconnected neurons.
For the network $\Phi$ we use the parameterization $\btheta = [\bw_1,\ldots,\bw_k,\bv]$ introduced in Section~\ref{sec:auxiliary}. Thus, we have  $\Phi(\btheta; \bx) = \sum_{l \in [k]} v_l \bw_l^\top \bx^l$.

By Theorem~\ref{thm:known KKT}, gradient flow converges in direction to $\tilde{\btheta} = [\tilde{\bw}_1,\ldots,\tilde{\bw}_k,\tilde{\bv}]$ which satisfies the KKT conditions of Problem~\ref{eq:optimization problem}. Thus, there are $\lambda_1,\ldots,\lambda_n$ such that for every $j \in [k]$ we have
\begin{equation}
\label{eq:depth 2 linear known kkt}
	\tilde{\bw}_j = \sum_{i \in [n]} \lambda_i \nabla_{\bw_j} \left( y_i \Phi(\tilde{\btheta}; \bx_i) \right) =  \sum_{i \in [n]} \lambda_i y_i \tilde{v}_j \bx_i^j~,
\end{equation}
and we have $\lambda_i \geq 0$ for all $i$, and $\lambda_i=0$ if $y_i \Phi(\tilde{\btheta}; \bx_i) = y_i \sum_{l \in [k]} \tilde{v}_l \tilde{\bw}_l^\top \bx_i^l \neq 1$.
By Theorem~\ref{thm:known KKT}, we also have $\lim_{t \to \infty} \norm{\btheta(t)} = \infty$. Hence, by Lemma~\ref{lem:balanced tilde} we have $\norm{\tilde{\bw}_j} = |\tilde{v}_j |$ for all $j \in [k]$.

Consider the following problem
\begin{equation}
\label{eq:depth 2 linear problem u}
\min \sum_{l \in [k]} \norm{\bu_l} \;\;\; \text{s.t. } \;\;\; \forall i \in [n] \;\; y_i \sum_{l \in [k]} \bu_l^\top \bx_i^l \geq 1~.
\end{equation}
For every $l \in [k]$ we denote $\tilde{\bu}_l = \tilde{v}_l \cdot \tilde{\bw}_l$. 
Since we assume that $\tilde{\bw}_l \neq \zero$ for every $l \in [k]$, and since $\norm{\tilde{\bw}_l} = |\tilde{v}_l |$, then $\tilde{\bu}_l \neq \zero$ for all $l \in [k]$.
Note that since 
%$\tilde{W},\tilde{\bv}$ 
$\tilde{\bw}_1,\ldots,\tilde{\bw}_k,\tilde{\bv}$ 
satisfy the constraints in Problem~\ref{eq:optimization problem}, then $\tilde{\bu}_1,\ldots,\tilde{\bu}_k$ satisfy the constraints in the above problem. In order to show that $\tilde{\bu}_1,\ldots,\tilde{\bu}_k$ satisfy the KKT condition of the problem, we need to prove that for every $j \in [k]$ we have
\begin{equation}
\label{eq:depth 2 linear kkt u}
	\frac{\tilde{\bu}_j}{\norm{\tilde{\bu}_j}} = \sum_{i \in [n]} \lambda'_i y_i \bx_i^j~
\end{equation}
for some  $\lambda'_i \geq 0$ such that $\lambda'_i = 0$ if $ y_i \sum_{l \in [k]} \tilde{\bu}_l^\top \bx_i^l \neq 1$.
From Eq.~\ref{eq:depth 2 linear known kkt} and since $\norm{\tilde{\bw}_l}=|\tilde{v}_l|$ for every $l \in [k]$, we have
\[
	\tilde{\bu}_j 
	= \tilde{v}_j \cdot \tilde{\bw}_j 
	= \tilde{v}_j \sum_{i \in [n]} \lambda_i y_i \tilde{v}_j \bx_i^j
	= \tilde{v}_j^2 \sum_{i \in [n]} \lambda_i y_i \bx_i^j
%	= \norm{\tilde{\bw}_j} \cdot  \sum_{i \in [n]} \lambda_i y_i \tilde{v}_j \bx_i^j
	= \norm{\tilde{v}_j \tilde{\bw}_j} \sum_{i \in [n]} \lambda_i y_i \bx_i^j
	= \norm{\tilde{\bu}_j} \sum_{i \in [n]} \lambda_i y_i \bx_i^j~.
\]
Note that we have $\lambda_i \geq 0$ for all $i$, and $\lambda_i=0$ if $y_i \sum_{l \in [k]} \tilde{\bu}_l^\top \bx_i^l  = y_i  \sum_{l \in [k]} \tilde{v}_l \tilde{\bw}_l^\top \bx_i^l \neq 1$.
Hence Eq.~\ref{eq:depth 2 linear kkt u} holds with $\lambda'_1,\ldots,\lambda'_n$ that satisfy the requirement. Since the objective in Problem~\ref{eq:depth 2 linear problem u} is convex and the constraints are affine functions, then its KKT condition is sufficient for global optimality. Namely, $\tilde{\bu}_1,\ldots,\tilde{\bu}_k$ are a global optimum for problem~\ref{eq:depth 2 linear problem u}.

We now deduce that $\tilde{\btheta}$ is a global optimum for Problem~\ref{eq:optimization problem}. Assume toward contradiction that there is a solution $\btheta'=[\bw'_1,\ldots,\bw'_k,\bv']$ for the constraints in Problem~\ref{eq:optimization problem} such that $\snorm{\btheta'}^2 < \snorm{\tilde{\btheta}}^2$. Let $\bu'_l = v'_l \bw'_l$. Note that the vectors $\bu'_l$ satisfy the constraints in Problem~\ref{eq:depth 2 linear problem u}. Moreover, we have
\[
	 \sum_{l \in [k]} \norm{\bu'_l}
	 = \sum_{l \in [k]} | v'_l | \cdot \norm{\bw'_l}
	 \leq \sum_{l \in [k]} \frac{1}{2}  \left( |v'_l |^2 + \norm{\bw'_l}^2 \right)
	 = \frac{1}{2} \norm{\btheta'}^2 
	 < \frac{1}{2} \norm{\tilde{\btheta}}^2
	 = \sum_{l \in [k]} \frac{1}{2}  \left( |\tilde{v}_l |^2 + \norm{\tilde{\bw}_l}^2 \right)~.
\] 
Since $\norm{\tilde{\bw}_l}=|\tilde{v}_l |$, the above equals
\[
	\sum_{l \in [k]} \norm{\tilde{\bw}_l}^2 
	= \sum_{l \in [k]}  |\tilde{v}_l | \cdot \norm{\tilde{\bw}_l} 
	= \sum_{l \in [k]} \norm{\tilde{\bu}_l}~,
\]
which contradicts the global optimality of $\tilde{\bu}_1,\ldots,\tilde{\bu}_k$.

\subsubsection{Proof of part 2}

Let $\{(\bx_i,y_i)\}_{i=1}^4$ be a dataset such that $y_i=1$ for all $i \in [4]$ and we have $\bx_1 = (0,1)^\top$, $\bx_2 = (1,0)^\top$, $\bx_3 = (0,-1)$ and $\bx_4 = (-1,0)$. Consider the initialization $\btheta(0)=[\bw_1(0),\bw_2(0),\bw_3(0),\bw_4(0),\bv(0)]$ such that $\bw_i(0) = 2 \bx_i$ and $v_i(0) = 2$ for every $i \in [4]$. Note that $\cl(\btheta(0)) = 4 \ell(4) < 1$ for both the exponential loss and the logistic loss, and therefore by Theorem~\ref{thm:known KKT} gradient flow converges in direction to a KKT point $\tilde{\btheta}$ of Problem~\ref{eq:optimization problem}, and we have $\lim_{t \to \infty} \cl(\btheta(t)) = 0$ and $\lim_{t \to \infty} \norm{\btheta(t)} = \infty$. 

We now show that for all $t \geq 0$ we have $\bw_i(t) = \alpha(t) \bx_i$ and $v_i(t) = \alpha(t)$ where $\alpha(t)>0$ and $\lim_{t \to \infty} \alpha(t) = \infty$.
Indeed, for such $\btheta(t)$, for every $j \in [4]$ we have
\begin{align*}
	-\frac{d \bw_j}{dt}
	&=\nabla_{\bw_j} \cl(\btheta)
	=  \sum_{i=1}^4 \ell'(y_i \Phi(\btheta; \bx_i)) \cdot y_i \nabla_{\bw_j}\Phi(\btheta; \bx_i) 
	=  \sum_{i=1}^4 \ell'\left( \sum_{l=1}^4 v_l \sigma(\bw_l^\top \bx_i)\right) \cdot  \left( v_j \sigma'(\bw_j^\top \bx_i) \bx_i \right)
	\\
	&= \ell'(\alpha^2) \cdot  \alpha \cdot  \sum_{i=1}^4   \sigma'(\bw_j^\top \bx_i) \bx_i 
	= \ell'(\alpha^2) \cdot \alpha \bx_j~,  
%	&=  \ell'(\Phi(\btheta; \bx_j)) \cdot  v_j \bx_j 
%	= \alpha \ell'(\alpha^2) \bx_j~.
\end{align*}
and 
\begin{align*}
	-\frac{d v_j}{dt}
	&= \nabla_{v_j} \cl(\btheta)
	=  \sum_{i=1}^4 \ell'(y_i \Phi(\btheta; \bx_i)) \cdot y_i \nabla_{v_j}\Phi(\btheta; \bx_i) 
	=  \sum_{i=1}^4 \ell'\left( \sum_{l=1}^4 v_l \sigma(\bw_l^\top \bx_i)\right) \cdot  \sigma(\bw_j^\top \bx_i)
	\\
	&=  \ell'(\alpha^2) \cdot \sum_{i=1}^4 \sigma(\bw_j^\top \bx_i)
	= \ell'( \alpha^2 ) \cdot \alpha~.
%	= \ell'( \Phi(\btheta; \bx_j) ) \cdot \bw_j^\top \bx_j 
%	= \alpha \ell'(\alpha^2)~. 
\end{align*}
Moreover, since $\lim_{t \to \infty} \norm{\btheta(t)} = \infty$ then $\lim_{t \to \infty} \alpha(t) = \infty$.

Hence, the KKT point $\tilde{\btheta}$ is such that for every $j \in [4]$ the vector $\tilde{\bw}_j$ points at the direction $\bx_j$, and we have $\tilde{v}_j=\norm{\tilde{\bw}_j}$. Also, the vectors $\tilde{\bw}_1,\tilde{\bw}_2,\tilde{\bw}_3,\tilde{\bw}_4$ have equal norms. That is, $\tilde{\bw}_j = \tilde{\alpha} \bx_j$ and $\tilde{v}_j = \tilde{\alpha}$ for some $\tilde{\alpha}>0$.
Moreover, since it satisfies the KKT condition of Problem~\ref{eq:optimization problem}, then we have
\[
	\tilde{\bw}_j 
	= \sum_{i = 1}^4 \lambda_i y_i \nabla_{\bw_j} \Phi(\tilde{\btheta};\bx_i)~,
\]
where $\lambda_i \geq 0$ and $\lambda_i = 0$ if $y_i \Phi(\tilde{\btheta};\bx_i) \neq 1$. Hence, there is $i$ such that $y_i \Phi(\tilde{\btheta};\bx_i) = 1$.Therefore, $\tilde{\alpha}^2 = 1$. Thus, we conclude that for all $j \in [4]$ we have $\tilde{\bw}_j = \bx_j$ and $\tilde{v}_j=1$.
Note that $\tilde{\bw}_j \neq \zero$ for all $j \in [4]$ as required.

Next, we show that $\tilde{\btheta}$ is not a local optimum of Problem~\ref{eq:optimization problem}. We show that for every $0 <\epsilon < 1$ there exists some $\btheta'$ such that $\norm{\btheta' - \tilde{\btheta}} \leq \epsilon$, $\btheta'$ satisfies the constraints of Problem~\ref{eq:optimization problem}, and $\snorm{\btheta'} < \snorm{\tilde{\btheta}}$.
Let $\epsilon' = \frac{\epsilon}{2\sqrt{2}}$. Let $\btheta'$ be such that $v'_j = \tilde{v}_j = 1$ for all $j \in [4]$, and we have $\bw'_1 = (\epsilon',1-\epsilon')^\top$, $\bw'_2 = (1-\epsilon',-\epsilon')^\top$, $\bw'_3 = (-\epsilon',-1+\epsilon')^\top$ and $\bw'_4 = (-1+\epsilon',\epsilon')^\top$. It is easy to verify that $\btheta'$ satisfies the constraints. Indeed, we have $\Phi(\btheta';\bx_i) = \left( (1-\epsilon') + \epsilon'  + 0 + 0 \right) =  1$. Also, we have $\norm{\btheta' - \tilde{\btheta}} = \sqrt{4 \cdot 2\epsilon'^2} = 2\sqrt{2} \epsilon' = \epsilon$. Finally,  
\[
	\snorm{\btheta'}^2
	= 4 \cdot \left(\epsilon'^2 + (1-\epsilon')^2 \right) + 4
	= 8 + 8\epsilon' \left( \epsilon' - 1 \right)
	< 8 
	= \snorm{\tilde{\btheta}}^2~.
\]

\subsection{Proof of Theorem~\ref{thm:depth 2 relu positive}}
\label{app:proof of depth 2 relu positive}

We assume w.l.o.g. that the second layer is fully-connected, namely, all hidden neurons are connected to the output neuron, since otherwise we can ignore disconnected neurons.
For the network $\Phi$ we use the parameterization $\btheta = [\bw_1,\ldots,\bw_k,\bv]$ introduced in Section~\ref{sec:auxiliary}. Thus, we have  $\Phi(\btheta; \bx) = \sum_{l \in [k]} v_l \sigma(\bw_l^\top \bx^l)$.

We denote $\tilde{\btheta} = [\tilde{\bw}_1,\ldots,\tilde{\bw}_k,\tilde{\bv}]$.
Since $\tilde{\btheta}$ is a KKT point of Problem~\ref{eq:optimization problem}, then there are $\lambda_1,\ldots,\lambda_n$ such that for every $j \in [k]$ we have
\begin{equation}
\label{eq:depth 2 relu known kkt}
	\tilde{\bw}_j = \sum_{i \in [n]} \lambda_i \nabla_{\bw_j} \left( y_i \Phi(\tilde{\btheta}; \bx_i) \right) = \sum_{i \in [n]} \lambda_i y_i \tilde{v}_j \sigma'(\tilde{\bw}_j^\top \bx_i^j) \bx_i^j~,
\end{equation}
and we have $\lambda_i \geq 0$ for all $i$, and $\lambda_i=0$ if $y_i \Phi(\tilde{\btheta}; \bx_i) = y_i \sum_{l \in [k]} \tilde{v}_l \sigma(\tilde{\bw}_l^\top \bx_i^l) \neq 1$.
Note that the KKT condition should be w.r.t. the Clarke subdifferential, but since for all $i,j$ we have $\tilde{\bw}_j^\top \bx_i^j \neq 0$ by our assumption, then we can use here the gradient. 
By Theorem~\ref{thm:known KKT}, we also have $\lim_{t \to \infty} \norm{\btheta(t)} = \infty$. Hence, by Lemma~\ref{lem:balanced tilde} we have $\norm{\tilde{\bw}_j} = |\tilde{v}_j |$ for all $j \in [k]$.

For $i \in [n]$ and $j \in [k]$ let $A_{ij} = \onefunc(\tilde{\bw}_j^\top \bx_i^j \geq 0)$. Consider the following problem
\begin{equation}
\label{eq:depth 2 relu problem u}
\min \sum_{l \in [k]} \norm{\bu_l} \;\;\; \text{s.t. } \;\;\; \forall i \in [n] \;\; y_i \sum_{l \in [k]}A_{il} \bu_l^\top \bx_i^l \geq 1~.
\end{equation}
For every $l \in [k]$ let $\tilde{\bu}_l = \tilde{v}_l \cdot \tilde{\bw}_l$. 
Since we assume that the inputs to all neurons in the computations $\Phi(\tilde{\btheta};\bx_i)$ are non-zero, then we must have $\tilde{\bw}_l \neq \zero$ for every $l \in [k]$. Since we also have $\norm{\tilde{\bw}_l} = |\tilde{v}_l |$, then $\tilde{\bu}_l \neq \zero$ for all $l \in [k]$.
Note that since $\tilde{\bw}_1,\ldots,\tilde{\bw}_k,\tilde{\bv}$ 
satisfy the constraints in Probelm~\ref{eq:optimization problem}, then $\tilde{\bu}_1,\ldots,\tilde{\bu}_k$ satisfy the constraints in the above problem. 
Indeed, for every $i \in [n]$ we have
\begin{equation*}
\label{eq:depth 2 relu u satisfies constraints}
	 y_i \sum_{l \in [k]}A_{il} \tilde{\bu}_l^\top \bx_i^l 
	 =  y_i \sum_{l \in [k]} \onefunc(\tilde{\bw}_l^\top \bx_i^l \geq 0) \tilde{v}_l \tilde{\bw}_l^\top \bx_i^l 
	 =  y_i \sum_{l \in [k]} \tilde{v}_l \sigma(\tilde{\bw}_l^\top \bx_i^l) 
	 \geq 1~.
\end{equation*}
In order to show that $\tilde{\bu}_1,\ldots,\tilde{\bu}_k$ satisfy the KKT condition of Probelm~\ref{eq:depth 2 relu problem u}, we need to prove that for every $j \in [k]$ we have
\begin{equation}
\label{eq:depth 2 relu kkt u}
	\frac{\tilde{\bu}_j}{\norm{\tilde{\bu}_j}} = \sum_{i \in [n]} \lambda'_i y_i A_{ij} \bx_i^j~
\end{equation}
for some $\lambda'_1,\ldots,\lambda'_n$ such that  for all $i$ we have $\lambda'_i \geq 0$, and $\lambda'_i = 0$ if $ y_i \sum_{l \in [k]} A_{il} \tilde{\bu}_l^\top \bx_i^l \neq 1$.
From Eq.~\ref{eq:depth 2 relu known kkt} and since $\norm{\tilde{\bw}_l}=|\tilde{v}_l|$ for every $l \in [k]$, we have
\[
	\tilde{\bu}_j 
	= \tilde{v}_j \cdot \tilde{\bw}_j 
	= \tilde{v}_j \sum_{i \in [n]} \lambda_i y_i  \tilde{v}_j A_{ij} \bx_i^j
	= \tilde{v}_j^2 \sum_{i \in [n]} \lambda_i y_i A_{ij} \bx_i^j
	= \norm{\tilde{v}_j \tilde{\bw}_j} \sum_{i \in [n]} \lambda_i y_i A_{ij}  \bx_i^j
	= \norm{\tilde{\bu}_j} \sum_{i \in [n]} \lambda_i y_i A_{ij}  \bx_i^j~.
\]
Note that we have $\lambda_i \geq 0$ for all $i$, and $\lambda_i=0$ if 
\[
	y_i \sum_{l \in [k]} A_{il} \tilde{\bu}_l^\top \bx_i^l  
	= y_i \sum_{l \in [k]} \tilde{v}_l \onefunc(\tilde{\bw}_l^\top \bx_i^l \geq 0) \tilde{\bw}_l^\top \bx_i^l  
	= y_i \sum_{l \in [k]} \tilde{v}_l \sigma(\tilde{\bw}_l^\top \bx_i^l)  
	\neq 1~.
\]
Hence Eq.~\ref{eq:depth 2 relu kkt u} holds with $\lambda'_1,\ldots,\lambda'_n$ that satisfy the requirement. Since the objective in Problem~\ref{eq:depth 2 relu problem u} is convex and the constraints are affine functions, then its KKT condition is sufficient for global optimality. Namely, $\tilde{\bu}_1,\ldots,\tilde{\bu}_k$ are a global optimum for Problem~\ref{eq:depth 2 relu problem u}.

We now deduce that $\tilde{\btheta}$ is a local optimum for Problem~\ref{eq:optimization problem}. 
Since for every $i \in [n]$ and $l \in [k]$ we have $\tilde{\bw}_l^\top \bx_i^l \neq 0$, then there is $\epsilon>0$, such that for every $i,l$ and every $\bw'_l$ with $\norm{\bw'_l-\tilde{\bw}_l} \leq \epsilon$ we have 
%$\bw'^\top_l \bx_i^l \neq 0$ and 
$\onefunc(\tilde{\bw}_l^\top \bx_i^l \geq 0) = \onefunc(\bw'^\top_l \bx_i^l \geq 0)$.
Assume toward contradiction that there is a solution $\btheta'=[\bw'_1,\ldots,\bw'_k,\bv']$ for the constraints in Problem~\ref{eq:optimization problem} such that $\snorm{\btheta'-\tilde{\btheta}} \leq \epsilon$ and $\snorm{\btheta'}^2 < \snorm{\tilde{\btheta}}^2$. 
Note that we have $\norm{\bw'_l-\tilde{\bw}_l} \leq \epsilon$ for every $l \in [k]$.
We denote $\bu'_l = v'_l \bw'_l$. 
The vectors $\bu'_1,\ldots,\bu'_k$ satisfy the constraints in Problem~\ref{eq:depth 2 relu problem u}, since we have
\begin{align*}
	 y_i \sum_{l \in [k]}A_{il} \bu_l'^\top \bx_i^l 
	 &=  y_i \sum_{l \in [k]} \onefunc(\tilde{\bw}_l^\top \bx_i^l \geq 0) v'_l \bw'^\top_l \bx_i^l 
	 = y_i \sum_{l \in [k]} \onefunc(\bw'^\top_l \bx_i^l \geq 0) v'_l \bw'^\top_l \bx_i^l 
	 \\
	 &=  y_i \sum_{l \in [k]} v'_l \sigma(\bw'^\top_l \bx_i^l ) 
	 \geq 1~,
\end{align*}
where the last inequality is since $\btheta'$ satisfies the constraints in Probelm~\ref{eq:optimization problem}.
Moreover, we have
\[
	 \sum_{l \in [k]} \norm{\bu'_l}
	 = \sum_{l \in [k]} | v'_l | \cdot \norm{\bw'_l}
	 \leq \sum_{l \in [k]} \frac{1}{2}  \left( |v'_l |^2 + \norm{\bw'_l}^2 \right)
	 = \frac{1}{2} \norm{\btheta'}^2 
	 < \frac{1}{2} \norm{\tilde{\btheta}}^2
	 = \sum_{l \in [k]} \frac{1}{2}  \left( |\tilde{v}_l |^2 + \norm{\tilde{\bw}_l}^2 \right)~.
\] 
Since $\norm{\tilde{\bw}_l}=|\tilde{v}_l |$, the above equals
\[
	\sum_{l \in [k]} \norm{\tilde{\bw}_l}^2 
	= \sum_{l \in [k]}  |\tilde{v}_l | \cdot \norm{\tilde{\bw}_l} 
	= \sum_{l \in [k]} \norm{\tilde{\bu}_l}~,
\]
which contradicts the global optimality of $\tilde{\bu}_1,\ldots,\tilde{\bu}_k$.

It remains to show that $\tilde{\btheta}$ may not be a global optimum of Problem~\ref{eq:optimization problem}, even if the network $\Phi$ is fully connected.
The following lemma concludes the proof.

\begin{lemma}
\label{lem:ReLU not global}
	Let $\Phi$ be a depth-$2$ fully-connected ReLU network with input dimension $2$ and two hidden neurons.
	Consider minimizing either the exponential or the logistic loss using gradient flow.
	Then, there exists a dataset $\{(\bx_i,y_i)\}_{i=1}^n$ and an initialization $\btheta(0)$, such that gradient flow converges to zero loss, converges in direction to a KKT point $\tilde{\btheta} = [\tilde{\bw}_1,\tilde{\bw}_2,\tilde{\bv}]$ of Problem~\ref{eq:optimization problem} such that $\inner{\tilde{\bw}_j,\bx_i} \neq 0$ for all $j \in \{1,2\}$ and $i \in [n]$, and $\tilde{\btheta}$ is not a global optimum.
\end{lemma}
\begin{proof}
Let $\bx_1 = \left(1,\frac{1}{4}\right)^\top$, $\bx_2 = \left(-1,\frac{1}{4}\right)^\top$, $\bx_3 = (0,-1)$, $y_1=y_2=y_3=1$. Let $\{(\bx_1,y_1),(\bx_2,y_2),(\bx_3,y_3)\}$ be a dataset. Consider the initialization $\btheta(0)$ such that $\bw_1(0) = (0,3)$, $v_1(0) = 3$, $\bw_2(0) = (0,-2)$ and $v_2(0) = 2$. Note that $\cl(\btheta(0)) = 2 \ell \left(\frac{9}{4}\right) + \ell(4) < 1$ for both the exponential loss and the logistic loss, 
and therefore by Theorem~\ref{thm:known KKT} gradient flow converges in direction to a KKT point $\tilde{\btheta}$ of Problem~\ref{eq:optimization problem}. %and we have $\lim_{t \to \infty}\cl(\btheta(t))=0$ and $\lim_{t \to \infty}\norm{\btheta(t)}=\infty$. 

Note that for $\btheta$ such that $\bw_1 = \alpha \cdot (0,1)^\top$ and $\bw_2 = \beta \cdot (0,-1)^\top$ for some $\alpha,\beta > 0$, and $v_1,v_2 > 0$, we have
\begin{align*}
	\nabla_{\bw_1} \cl(\btheta) 
	&= \sum_{i=1}^3 \ell'(y_i \Phi(\btheta; \bx_i)) \cdot y_i \nabla_{\bw_1}\Phi(\btheta; \bx_i) 
	= \sum_{i=1}^3 \ell'\left(v_1 \sigma(\bw_1^\top \bx_i) + v_2 \sigma(\bw_2^\top \bx_i)\right) \cdot v_1 \sigma'(\bw_1^\top \bx_i) \bx_i 
	%= \sum_{i=1}^3 \ell'(y_i \Phi(\btheta; \bx_i)) \cdot v_1 \sigma'(\bw_1^\top \bx_i) \bx_i 
	\\
	%&= \sum_{i=1}^2 e^{-v_1 \sigma(\bw_1^\top \bx_i) } \left(- v_1 \bx_i \right)
	&= \sum_{i=1}^2 \ell'(v_1 \sigma(\bw_1^\top \bx_i) ) \cdot v_1 \bx_i 
	%= -v_1 e^{-v_1 \frac{\alpha}{4} }  \sum_{i=1}^2 \bx_i~,
	= v_1 \ell'\left(v_1 \frac{\alpha}{4} \right)  \sum_{i=1}^2 \bx_i~,
\end{align*}
and 
\begin{align*}
	\nabla_{v_1} \cl(\btheta)
	= \sum_{i=1}^3 \ell'(y_i \Phi(\btheta; \bx_i)) \cdot y_i \nabla_{v_1}\Phi(\btheta; \bx_i)
	%= \sum_{i=1}^3 \ell'(v_1 \sigma(\bw_1^\top \bx_i) + v_2 \sigma(\bw_2^\top \bx_i)) \cdot \sigma(\bw_1^\top \bx_i)~.
	= \sum_{i=1}^3 \ell'(y_i \Phi(\btheta; \bx_i)) \cdot \sigma(\bw_1^\top \bx_i)~.
%	\\
%	&= - \sum_{i=1}^2 e^{-v_1 \sigma(\bw_1^\top \bx_i)} \sigma(\bw_1^\top \bx_i)~.
\end{align*}
Hence, $-\nabla_{\bw_1} \cl(\btheta)$ points in the direction $(0,1)^\top$ and $-\nabla_{v_1} \cl(\btheta) > 0$. 
Moreover, we have
\begin{align*}
	\nabla_{\bw_2} \cl(\btheta) 
	&= \sum_{i=1}^3 \ell'(y_i \Phi(\btheta; \bx_i)) \cdot y_i \nabla_{\bw_2 }\Phi(\btheta; \bx_i) 
	= \sum_{i=1}^3 \ell'(v_1 \sigma(\bw_1^\top \bx_i) + v_2 \sigma(\bw_2^\top \bx_i)) \cdot v_2 \sigma'(\bw_2^\top \bx_i) \bx_i 
	\\
	&= \ell'(v_2 \sigma(\bw_2^\top \bx_3) ) \cdot v_2 \bx_3 
	= v_2 \ell'(v_2 \beta ) \bx_3~,
\end{align*}
and 
\begin{align*}
	\nabla_{v_2} \cl(\btheta)
	&= \sum_{i=1}^3 \ell'(y_i \Phi(\btheta; \bx_i)) \cdot y_i \nabla_{v_2}\Phi(\btheta; \bx_i) 
	= \sum_{i=1}^3 \ell'(v_1 \sigma(\bw_1^\top \bx_i) + v_2 \sigma(\bw_2^\top \bx_i)) \cdot \sigma(\bw_2^\top \bx_i) 
	\\
	&=  \ell'(v_2 \sigma(\bw_2^\top \bx_3)) \sigma(\bw_2^\top \bx_3)
	= \ell'(v_2 \beta) \cdot \beta~.
\end{align*}
Therefore, $-\nabla_{\bw_2} \cl(\btheta)$ points in the direction $(0,-1)^\top$ and $-\nabla_{v_2} \cl(\btheta) > 0$. 
Hence for every $t$ we have $\bw_1(t) = \alpha(t) \cdot (0,1)^\top$ for some $\alpha(t)>0$ and $v_1(t)>0$. Also, we have  $\bw_2(t) = \beta(t) \cdot (0,-1)^\top$ for some $\beta(t)>0$ and $v_2(t)>0$.
By Lemma~\ref{lem:from du fully connected}, we have for every $t \geq 0$ that  $\norm{\bw_1(t)}^2 - v_1(t)^2= \norm{\bw_1(0)}^2 - v_1(0)^2=0$ and $\norm{\bw_2(t)}^2 - v_2(t)^2= \norm{\bw_2(0)}^2 - v_2(0)^2=0$. Hence, we have $v_1(t) = \alpha(t)$ and $v_2(t) = \beta(t)$.
Therefore, we have $\tilde{\bw}_1 = \tilde{\alpha} \cdot (0,1)^\top$ and $\tilde{v}_1 = \tilde{\alpha}$ for some $\tilde{\alpha} \geq 0$. Likewise, we have $\tilde{\bw}_2 = \tilde{\beta} \cdot (0,-1)^\top$ and $\tilde{v}_2 = \tilde{\beta}$ for some $\tilde{\beta} \geq 0$. 
Since $\tilde{\btheta}$ satisfies the constraints in Probelm~\ref{eq:optimization problem}, then $\tilde{\alpha} \geq 2$ and $\tilde{\beta} \geq 1$. 
%Hence $\snorm{\tilde{\btheta}}^2 \geq 4+4+1+1 = 10$. 
Note that $\inner{\tilde{\bw}_j,\bx_i} \neq 0$ for all $j \in \{1,2\}$ and $i \in \{1,2,3\}$.

We now show that there exists a solution $\btheta'$ to Problem~\ref{eq:optimization problem} with a smaller norm, and hence $\tilde{\btheta}$ is not a global optimum. Let $\btheta' = [\bw'_1,\bw'_2,\bv']$ such that 
%$\bw'_1 = \frac{\bx_1}{\norm{\bx_1}}$, $v'_1 = 1$, $\bw'_2 = \left( - \frac{5}{4}, -1 \right)$, $v'_2 = 1$. 
$\bw'_1 = \frac{\bx_1}{\tilde{\alpha} \norm{\bx_1}}$, $v'_1 = \tilde{\alpha}$, $\bw'_2 = \frac{1}{\tilde{\beta}} \cdot \left( - \frac{5}{4}, -1 \right)$, and $v'_2 = \tilde{\beta}$. 
It is easy to verify that $\btheta'$ satisfies the constraints in Problem~\ref{eq:optimization problem}, and we have 
%$\snorm{\btheta'}^2 = 1+1+  \left( \frac{25}{16} + 1 \right) + 1 < 10 \leq \snorm{\tilde{\btheta}}^2$.
\[
	\snorm{\btheta'}^2 
	= \frac{1}{\tilde{\alpha}^2} +\tilde{\alpha}^2 + \frac{1}{\tilde{\beta}^2} \left( \frac{25}{16} + 1 \right) + \tilde{\beta}^2 
	< \frac{1}{4} +\tilde{\alpha}^2 + 1 \cdot  3 + \tilde{\beta}^2 
	< \tilde{\beta}^2  + \tilde{\alpha}^2 + \tilde{\alpha}^2  +  \tilde{\beta}^2 
	=  \snorm{\tilde{\btheta}}^2~.
\]
\end{proof}

\subsection{Proof of Theorem~\ref{thm:depth 2 cnn}}
\label{app:proof of depth 2 cnn}

Let $\bx = \left(4,\frac{1}{\sqrt{2}},-4, \frac{1}{\sqrt{2}} \right)^\top$ and $y = 1$. Let $\btheta(0) = [\bw(0),\bv(0)]$ where $\bw(0) = (0,1)^\top$ and $\bv(0) = 
\left( \frac{1}{\sqrt{2}}, \frac{1}{\sqrt{2}} \right)^\top$. Note that $\Psi(\btheta(0); \bx) = 1$ and hence $\cl(\btheta(0)) < 1$ for both the exponential loss and the logistic loss. Therefore, by Theorem~\ref{thm:known KKT} gradient flow converges in direction to a KKT point $\tilde{\btheta}$ of Problem~\ref{eq:optimization problem}, and we have $\lim_{t \to \infty} \cl(\btheta(t)) = 0$ and $\lim_{t \to \infty}\norm{\btheta(t)} = \infty$.

The symmetry of the input $\bx$ and the initialization $\btheta(0)$ implies that the direction of $\bw$ does not change during the training, and that we have $v_1(t)=v_2(t)>0$ for all $t \geq 0$. More formally, this claim follows from the following calculation. For $j \in \{1,2\}$ we have
\[
	\nabla_{v_j} \cl(\btheta) = 
	%e^{-y \Phi(\btheta; \bx)} \left( -y \nabla_{v_j} \Phi(\btheta;\bx) \right)
	\ell'(y \Phi(\btheta; \bx) ) \cdot y \nabla_{v_j} \Phi(\btheta;\bx)
	%= e^{-y \Phi(\btheta; \bx)} \left( - \sigma(\bw^\top \bx^{(j)}) \right)~.
	=\ell'(y \Phi(\btheta; \bx) ) \cdot  \sigma(\bw^\top \bx^{(j)})~.
\]
Moreover,
\[
	\nabla_\bw \cl(\btheta) = 
	%e^{-y \Phi(\btheta; \bx)} \left( -y \nabla_{\bw} \Phi(\btheta;\bx) \right)
	\ell'(y \Phi(\btheta; \bx)) \cdot y \nabla_{\bw} \Phi(\btheta;\bx) 
	%= - e^{-y \Phi(\btheta; \bx)} \left( v_1 \sigma'(\bw^\top \bx^{(1)}) \bx^{(1)} + v_2 \sigma'(\bw^\top \bx^{(2)}) \bx^{(2)}\right)~.
	=\ell'(y \Phi(\btheta; \bx)) \cdot  \left( v_1 \sigma'(\bw^\top \bx^{(1)}) \bx^{(1)} + v_2 \sigma'(\bw^\top \bx^{(2)}) \bx^{(2)}\right)~.
\]
Hence, if $v_1=v_2>0$ and $\bw$ points in the direction $(0,1)^\top$, then it is easy to verify that $\nabla_{v_1} \cl(\btheta)=\nabla_{v_1} \cl(\btheta)<0$ and that $\nabla_\bw \cl(\btheta)$ points in the direction of  $-(\bx^{(1)}+\bx^{(2)}) = -(0,\sqrt{2})^\top$.
Furthermore, by Lemma~\ref{lem:from du sparse and cnn}, for every $t \geq 0$ we have $\norm{\bw(t)}^2 - \norm{\bv(t)}^2 = \norm{\bw(0)}^2 - \norm{\bv(0)}^2 = 0$. 
%Since for every $t$ we have $\norm{\bw(t)}>0$ then we also have $0 < \norm{\bv(t)} = \sqrt{v_1(t)^2 + v_2(t)^2} = \sqrt{2v_1}$. Hence, $v_1(t)>0$ for all $t$, and similarly for $v_2(t)$.

Therefore, the KKT point $\tilde{\btheta}=[\tilde{\bw},\tilde{\bv}]$ is such that $\tilde{\bw}$ points at the direction $(0,1)^\top$, $\tilde{v}_1=\tilde{v}_2>0$, and $\norm{\tilde{\bw}}=\norm{\tilde{\bv}}$. Since $\tilde{\btheta}$ satisfies the KKT conditions of Problem~\ref{eq:optimization problem}, then we have
\[
	\tilde{\bw} 
	= \lambda \nabla_{\bw} \left( y \Phi(\tilde{\btheta};\bx) \right)~, 
	%= \lambda \left( \tilde{v}_1 \sigma'(\tilde{\bw}^\top \bx^{(1)}) \bx^{(1)} + \tilde{v}_2 \sigma'(\tilde{\bw}^\top \bx^{(2)}) \bx^{(2)} \right)~,
\]
 where $\lambda \geq 0$ and $\lambda=0$ if $y \Phi(\tilde{\btheta}; \bx) \neq 1$. 
 %Note that since $\tilde{\bw}$ points at the direction $(0,1)^\top$ then $\tilde{\bw}^\top \bx^{(1)} \neq 0$ and therefore the ReLU function is differentiable at $\tilde{\bw}^\top \bx^{(1)}$. Thus, we use here the gradient and not the Clarke subdifferential.
 Hence, we must have $y \Phi(\tilde{\btheta}; \bx) = 1$. Letting $z:=\tilde{v}_1=\tilde{v}_2$ and using $2z^2 =  \norm{\tilde{\bv}}^2 = \norm{\tilde{\bw}}^2 = \tilde{w}_2^2$, we have
 \begin{align*}
 	1 
	&= \tilde{v}_1 \sigma(\tilde{\bw}^\top \bx^{(1)}) + \tilde{v}_2 \sigma(\tilde{\bw}^\top \bx^{(2)})
	= z \tilde{\bw}^\top \bx^{(1)} + z \tilde{\bw}^\top \bx^{(2)}
	= z \tilde{w}^\top \left(  \bx^{(1)} +  \bx^{(2)} \right)
	= z\cdot \tilde{w}_2 \sqrt{2}
	\\
	&= \frac{ \tilde{w}_2} {\sqrt{2}} \cdot \tilde{w}_2 \sqrt{2}
	=  \tilde{w}_2^2~.
 \end{align*}
 Therefore, $\tilde{\bw} = (0,1)^\top$ and $\tilde{\bv} = \left( \frac{1}{\sqrt{2}}, \frac{1}{\sqrt{2}} \right)$.
 Note that we have $\inner{\tilde{\bw},\bx^{(1)}} \neq 0$ and $\inner{\tilde{\bw},\bx^{(2)}} \neq 0$. 
 
 It remains to show that $\tilde{\btheta}$ is not a local optimum of Problem~\ref{eq:optimization problem}. We show that for every $0<\epsilon'<1$ there exists some $\btheta'=[\bw',\bv']$ such that $\norm{\btheta' - \tilde{\btheta}} \leq \epsilon'$, $\btheta'$ satisfies the constrains in Problem~\ref{eq:optimization problem}, and $\snorm{\btheta'}<\snorm{\tilde{\btheta}}$.
Let $\epsilon = \frac{\epsilon'^2}{2} \in (0, 1/2)$, and let $\bw' = (\sqrt{\epsilon},1-\epsilon)^\top$ and $\bv' = \left( \frac{1}{\sqrt{2}} + \frac{\sqrt{\epsilon}}{2}, \frac{1}{\sqrt{2}} - \frac{\sqrt{\epsilon}}{2} \right)^\top$. Note that 
\begin{align*}
	\norm{\btheta'}^2
	&= \norm{\bw'}^2 + \norm{\bv'}^2
	= \epsilon + (1-\epsilon)^2 +  \left(\frac{1}{\sqrt{2}} + \frac{\sqrt{\epsilon}}{2}\right)^2 + \left(\frac{1}{\sqrt{2}} - \frac{\sqrt{\epsilon}}{2}\right)^2
	= \epsilon + 1 + \epsilon^2 -2\epsilon + 1 + \frac{\epsilon}{2} 
	\\
	&= 2 - \frac{\epsilon}{2} + \epsilon^2 
	<  2 - \frac{\epsilon}{2} + \frac{\epsilon}{2} 
	= 2 
	= \norm{\tilde{\bw}}^2 + \norm{\tilde{\bv}}^2
	= \norm{\tilde{\btheta}}^2~. 	
\end{align*}
Moreover,
\begin{align*}
	\norm{\btheta' - \tilde{\btheta}}^2
	= \epsilon + \epsilon^2 + \frac{\epsilon}{4} + \frac{\epsilon}{4}
	= \epsilon^2 + \frac{3 \epsilon}{2}
	=  \frac{\epsilon'^4}{4} +  \frac{3 \epsilon'^2}{4}
	< \frac{\epsilon'^2}{4} +  \frac{3 \epsilon'^2}{4}
	= \epsilon'^2~.
\end{align*}
Finally, we show that $\btheta'$ satisfies the constraints:
\begin{align*}
	\Phi(\btheta'; \bx)
	&= v'_1 \sigma(\bw'^\top \bx^{(1)}) +  v'_2 \sigma(\bw'^\top \bx^{(2)})
	\\
	&= \left( \frac{1}{\sqrt{2}} + \frac{\sqrt{\epsilon}}{2} \right) \left(4 \sqrt{\epsilon} + \frac{1}{\sqrt{2}} \cdot (1-\epsilon) \right) + 
	      \left( \frac{1}{\sqrt{2}} - \frac{\sqrt{\epsilon}}{2} \right) \left(-4 \sqrt{\epsilon} + \frac{1}{\sqrt{2}} \cdot (1-\epsilon) \right)
	\\
	&= \frac{1}{\sqrt{2}} \cdot (1-\epsilon) \left( \frac{1}{\sqrt{2}} + \frac{\sqrt{\epsilon}}{2} + \frac{1}{\sqrt{2}} - \frac{\sqrt{\epsilon}}{2}  \right) 
		+ 4 \sqrt{\epsilon} \left(  \frac{1}{\sqrt{2}} + \frac{\sqrt{\epsilon}}{2} - \frac{1}{\sqrt{2}} + \frac{\sqrt{\epsilon}}{2} \right) 
	\\
	&= 1-\epsilon + 4 \epsilon
	= 1 + 3 \epsilon
	\geq 1~.
\end{align*}

\subsection{Proof of Theorem~\ref{thm:deep negative}}
\label{app:proof of deep negative}

Let $\bx=(1,1)^\top$ and $y=1$. Consider the initialization $\btheta(0) = [\bw_1(0),\ldots,\bw_m(0)]$, where $\bw_j(0) = (1,1)^\top$ for every $j \in [m]$. Note that $\cl(\btheta(0)) = \ell(2) < 1$ for both linear and ReLU networks with the exponential loss or the logistic loss, and therefore by Theorem~\ref{thm:known KKT} gradient flow converges in direction to a KKT point $\tilde{\btheta}$ of Problem~\ref{eq:optimization problem}, and we have  $\lim_{t \to \infty}\cl(\btheta(t))=0$ and $\lim_{t \to \infty}\norm{\btheta(t)}=\infty$. It remains to show that it does not converge in direction to a local optimum of Problem~\ref{eq:optimization problem}.

From the symmetry of the network $\Phi$ and the initialization $\btheta(0)$, it follows that for all $t$ the network $\Phi(\btheta(t); \cdot)$ remains symmetric, namely, there are $\alpha_j(t)$ such that $\bw_j(t) = (\alpha_j(t),\alpha_j(t))$. Moreover, by Lemma~\ref{lem:from du sparse and cnn}, for every $t \geq 0$ and $j,l \in [m]$ we have $\alpha_j(t) = \alpha_l(t) := \alpha(t)$. Thus, gradient flow converges in direction to the KKT point $\tilde{\btheta} = [\tilde{\bw}_1,\ldots,\tilde{\bw}_m]$ such that $\tilde{\bw}_j = \left(2^{-1/m},2^{-1/m}\right)^\top$ for all $j \in [m]$. Note that since the dataset is of size $1$, then every KKT point of Problem~\ref{eq:optimization problem} must label the input $\bx$ with exactly $1$.

We now show that $\tilde{\btheta}$ is not a local optimum of Problem~\ref{eq:optimization problem}. 
The following arguments hold for both linear and ReLU networks.
Let $0 < \epsilon < \frac{1}{2}$. Let $\btheta' = [\bw'_1,\ldots,\bw'_m]$ such that for every $j \in [m]$ we have $\bw'_j = \left(\left(\frac{1+\epsilon}{2}\right)^{1/m}, \left(\frac{1-\epsilon}{2}\right)^{1/m} \right)^\top$. We have
\[
	y \cdot \Phi(\btheta'; \bx)
	= \left(\frac{1+\epsilon}{2}\right) + \left(\frac{1-\epsilon}{2}\right)
	= 1~.
\]
Hence, $\btheta'$ satisfies the constraints in Problem~\ref{eq:optimization problem}. We now show that for every sufficiently small $\epsilon>0$ we have $\snorm{\btheta'}^2 < \snorm{\tilde{\btheta}}^2$.
We need to show that
\[
	m \left(\frac{1+\epsilon}{2}\right)^{2/m} + m \left(\frac{1-\epsilon}{2}\right)^{2/m} < 2m \left(\frac{1}{2} \right)^{2/m}~.
\]
Therefore, it suffices to show that 
\[
	 \left(1+\epsilon\right)^{2/m} +  \left(1-\epsilon\right)^{2/m} < 2~.
\]
Let $g:\reals \to \reals$ such that $g(s) =  \left(1+s\right)^{2/m} +  \left(1-s\right)^{2/m}$. We have $g(0) = 2$. The derivatives of $g$ satisfy
\[
	g'(s)
	= \frac{2}{m}  \left(1+s\right)^{\frac{2}{m} - 1} -  \frac{2}{m}  \left(1 - s\right)^{\frac{2}{m} - 1}~,
\]
and
\[
	g''(s) 
	= \frac{2}{m} \left( \frac{2}{m} - 1 \right)  \left(1+s\right)^{\frac{2}{m} - 2} +  \frac{2}{m}  \left( \frac{2}{m} - 1 \right) \left(1 -s\right)^{\frac{2}{m} - 2}~.
\]
Since $m \geq 3$ we have $g'(0)=0$ and $g''(0) < 0$. Hence, $0$ is a local maximum of $g$. Therefore for every sufficiently small $\epsilon>0$ we have $g(\epsilon)<2$ and thus $\snorm{\btheta'}^2 < \snorm{\tilde{\btheta}}^2$.

Finally, note that the inputs to all neurons in the computation $\Phi(\tilde{\btheta}; \bx)$ are positive.

\subsection{Proof of Theorem~\ref{thm:positive each layer linear}}
\label{app:proof of positive each layer linear}

By Theorem~\ref{thm:known KKT} gradient flow converge in direction to a KKT point $\tilde{\btheta}=[\tilde{\bu}^{(l)}]_{l=1}^m$ of Problem~\ref{eq:optimization problem}. We now show that for every layer $l \in [m]$ the parameters vector $\tilde{\bu}^{(l)}$ is a global optimum of Problem~\ref{eq:optimization problem one layer} w.r.t. $\tilde{\btheta}$.

%Given parameters $\btheta$, for $i \in [n]$ and $l \in [m] \cup \{0\}$ we denote by $\bx_i^{(l)} \in \reals^{d_l}$ the output of the $l$-th layer in the computation $\Phi(\btheta;\bx_i)$, where $\bx_i^{(0)}=\bx_i$.
Since $\tilde{\btheta}$ is a KKT point of Problem~\ref{eq:optimization problem}, then there are $\lambda_1,\ldots,\lambda_n$ such that for every $l \in [m]$ we have
\begin{equation*}
\label{eq:deep linear known kkt}
	\tilde{\bu}^{(l)} 
	= \sum_{i \in [n]} \lambda_i  \frac{\partial \left( y_i \Phi(\tilde{\btheta};\bx_i) \right) }{\partial \bu^{(l)} }~,  
	%= \frac{\partial}{\partial \bu^{(l)} } \left[ \sum_{i \in [n]} \lambda_i y_i \tilde{W}^{(m)} \cdot \ldots \cdot \tilde{W}^{(1)} \bx_i \right]~,
\end{equation*}
%where $\tilde{W}^{(j)}$ are the weight matrices that correspond to $\tilde{\btheta}$, and we have $\lambda_i \geq 0$ for all $i$, and $\lambda_i=0$ if $y_i \tilde{W}^{(m)} \ldots \tilde{W}^{(1)} \bx_i \neq 1$. 
where $\lambda_i \geq 0$ for all $i$, and $\lambda_i=0$ if $y_i \Phi(\tilde{\btheta};\bx_i)  \neq 1$. 
Letting $\btheta'(\bu^{(l)}) =[\tilde{\bu}^{(1)},\ldots,\tilde{\bu}^{(l-1)},\bu^{(l)},\tilde{\bu}^{(l+1)},\ldots,\tilde{\bu}^{(m)}]$, the above equation can be written as
\[
	\tilde{\bu}^{(l)} 
	= \sum_{i \in [n]} \lambda_i  \frac{\partial \left( y_i \Phi(\btheta'(\tilde{\bu}^{(l)});\bx_i) \right) }{\partial \bu^{(l)} }~,
\]
where $\lambda_i \geq 0$ for all $i$, and $\lambda_i=0$ if $y_i \Phi(\btheta'(\tilde{\bu}^{(l)}); \bx_i) = y_i \Phi(\tilde{\btheta}; \bx_i)  \neq 1$. 
Moreover, if the constraints in Problem~\ref{eq:optimization problem} are satisfies in $\tilde{\btheta}$, then the constrains in Problem~\ref{eq:optimization problem one layer} are also satisfied for every $l \in [m]$ in $\tilde{\bu}^{(l)}$ w.r.t. $\tilde{\btheta}$. Hence, for every $l \in [m]$ the KKT conditions of Problem~\ref{eq:optimization problem one layer} w.r.t. $\tilde{\btheta}$ hold. Since the constraints in Problem~\ref{eq:optimization problem one layer} are affine and the objective is convex, then this KKT point is a global optimum.

\subsection{Proof of Theorem~\ref{thm:negative each layer relu}}
\label{app:proof of negative each layer relu}

Let $\{(\bx_i,y_i)\}_{i=1}^4$ be a dataset such that $y_i=1$ for all $i \in [4]$ and we have $\bx_1 = (0,1)^\top$, $\bx_2 = (1,0)^\top$, $\bx_3 = (0,-1)$ and $\bx_4 = (-1,0)$. In the proof of Theorem~\ref{thm:depth 2 linear} (part 2) we showed that for an appropriate initialization, for both the exponential loss and the logistic loss gradient flow converges to zero loss, and converges in direction to a KKT point $\tilde{\btheta}$ of Problem~\ref{eq:optimization problem}. Moreover, in the proof of Theorem~\ref{thm:depth 2 linear} we showed that the KKT point $\tilde{\btheta}$ is such that for all $j \in [4]$ we have $\tilde{\bw}_j = \bx_j$ and $\tilde{v}_j = 1$.

We show that $\tilde{\bw}_1,\tilde{\bw}_2,\tilde{\bw}_3,\tilde{\bw}_4$ is not a local optimum of Problem~\ref{eq:optimization problem one layer} w.r.t. $\tilde{\btheta}$. It suffices to prove that for every $0 <\epsilon < 1$ there exists some $\btheta'$ such that $v'_j = \tilde{v}_j$ for all $j \in [4]$, $\snorm{\btheta' - \tilde{\btheta}} \leq \epsilon$, $\btheta'$ satisfies the constraints, and $\snorm{\btheta'} < \snorm{\tilde{\btheta}}$. The existence of such $\btheta'$ is shown in the proof of Theorem~\ref{thm:depth 2 linear}. Hence, we conclude the proof of the theorem.

\subsection{Proof of Theorem~\ref{thm:positive each layer relu}}
\label{app:proof of positive each layer relu}

By Theorem~\ref{thm:known KKT} gradient flow converge in direction to a KKT point $\tilde{\btheta}=[\tilde{\bu}^{(l)}]_{l=1}^m$ of Problem~\ref{eq:optimization problem}. Let $l \in [m]$ and assume that for every $i \in [n]$ the inputs to all neurons in layers $l,\ldots,m-1$ in the computation $\Phi(\tilde{\btheta};
	\bx_i)$ are non-zero. We now show that the parameters vector $\tilde{\bu}^{(l)}$ is a local optimum of Problem~\ref{eq:optimization problem one layer} w.r.t. $\tilde{\btheta}$. 

For $i \in [n]$ and $k \in [m-1]$ we denote by $\bx_i^{(k)} \in \reals^{d_k}$ the output of the $k$-th layer in the computation $\Phi(\tilde{\btheta};\bx_i)$, and denote $\bx_i^{(0)}=\bx_i$. 
If $l \in [m-1]$ then we define the following notations.
We denote by $f_l: \reals^{d_{l}} \to \reals$ the function computed by layers $l+1,\ldots,m$ of $\Phi(\tilde{\btheta};\cdot)$. Thus, we have $\Phi(\tilde{\btheta};\bx_i) = f_l(\bx_i^{(l)}) =  f_l \circ \sigma \left(\tilde{W}^{(l)} \bx_i^{(l-1)}\right)$, where $\tilde{W}^{(l)}$ is the weight matrix that corresponds to $\tilde{\bu}^{(l)}$.
For $i \in [n]$ we denote by $h_i$ the function $\bu^{(l)} \mapsto f_l \circ \sigma (W^{(l)} \bx_i^{(l-1)})$ where $W^{(l)}$ is the weights matrix that corresponds to $\bu^{(l)}$. Thus, $\Phi(\tilde{\btheta};\bx_i) = h_i(\tilde{\bu}^{(l)})$.
If $l=m$ then we denote by $h_i$ the function $\bu^{(m)} \mapsto W^{(m)} \bx_i^{(m-1)}$, thus we also have $\Phi(\tilde{\btheta};\bx_i) = h_i(\tilde{\bu}^{(m)})$.

Since $\tilde{\btheta}$ is a KKT point of Problem~\ref{eq:optimization problem}, then there are $\lambda_1,\ldots,\lambda_n$ such that
\begin{equation*}
\label{eq:deep relu known kkt}
	\tilde{\bu}^{(l)} 
	= \sum_{i \in [n]} \lambda_i  \frac{\partial \left( y_i \Phi(\tilde{\btheta};\bx_i) \right) }{\partial \bu^{(l)} }
	= \sum_{i \in [n]} \lambda_i \frac{\partial}{\partial \bu^{(l)} } \left[ y_i \cdot  h_i(\tilde{\bu}^{(l)}) \right]~,
\end{equation*}
where $\lambda_i \geq 0$ for all $i$, and $\lambda_i=0$ if $ y_i \cdot h_i(\tilde{\bu}^{(l)})  \neq 1$.
Note that since the inputs to all neurons in layers $l,\ldots,m-1$ in the computation $\Phi(\tilde{\btheta}; \bx_i)$ are non-zero, then the function $h_i$ is differentiable at $\tilde{\bu}^{(l)}$. Therefore in the above KKT condition we use the derivative rather than the Clarke subdifferential.
Moreover, if the constraints in Problem~\ref{eq:optimization problem} are satisfies in $\tilde{\btheta}$, then the constrains in Problem~\ref{eq:optimization problem one layer} are also satisfied in $\tilde{\bu}^{(l)}$ w.r.t. $\tilde{\btheta}$. 
Hence, 
%for every $l \in [m]$ 
the KKT condition of Problem~\ref{eq:optimization problem one layer} w.r.t. $\tilde{\btheta}$ holds.

%For $i \in [n]$ we denote by $h_i$ the function $\bu^{(l)} \mapsto f_l \circ \sigma (W^{(l)} \bx_i^{(l-1)})$ where $W^{(l)}$ is the weights matrix that corresponds to $\bu^{(l)}$. 
Also, note that since the inputs to all neurons in layers $l,\ldots,m-1$ in the computation $\Phi(\tilde{\btheta}; \bx_i)$ are non-zero, then the function $h_i$ is locally linear near $\tilde{\bu}^{(l)}$. We denote this linear function by $\tilde{h}_i$. Therefore, $\tilde{\bu}^{(l)}$ is a KKT point of the following problem
\[
	\min_{\bu^{(l)}} \frac{1}{2} \norm{\bu^{(l)}}^2 \;\;\;\; \text{s.t. } \;\;\; \forall i \in [n] \;\;\; y_i \tilde{h}_i(\bu^{(l)}) \geq 1~.
\]
Since the constrains here are affine and the objective is convex, then $\tilde{\bu}^{(l)}$ is a global optimum of the above problem. Thus, there is a small ball near $\tilde{\bu}^{(l)}$ where $\tilde{\bu}^{(l)}$ is the optimum of Problem~\ref{eq:optimization problem one layer} w.r.t. $\tilde{\btheta}$, namely, it is a local optimum. 

Finally, note that in the proof of Lemma~\ref{lem:ReLU not global} the parameters vector $\btheta'$ is obtained from $\tilde{\btheta}$ by changing only the first layer. Hence, in ReLU networks gradient flow might converge in direction to a KKT point of Problem~\ref{eq:optimization problem} which is not a global optimum of Problem~\ref{eq:optimization problem one layer}, even if all inputs to neurons are non-zero. 

%stam

\end{document}